\def\year{2018}\relax

\documentclass[letterpaper]{article}
\usepackage{aaai19}  
\usepackage{times}  
\usepackage{helvet}  
\usepackage{courier}  
\usepackage{url}  
\usepackage{graphicx}  
\usepackage{bm}
\usepackage{amsmath}
\usepackage{amsthm}
\usepackage{amsfonts}
\usepackage{cases}
\usepackage[]{multicol}
\usepackage[]{multirow}
\usepackage{graphicx}
\usepackage{subfigure}
\usepackage{algorithm}
\usepackage{url}
\usepackage{algorithmic}
\usepackage{enumitem}

\newcommand{\citet}[1]
{\citeauthor{#1} \shortcite{#1}}
\newcommand{\citep}{\cite}

\usepackage{color}
\newcommand{\todo}[1]{}
\renewcommand{\todo}[1]{{\color{red} TODO: {#1}}}

\newtheorem{theorem}{Theorem}
\newtheorem{lemma}{Lemma}

\newcommand{\argmax}{\operatornamewithlimits{\mathrm{arg\,max}}}
\newcommand{\algo}{{\sc\textsf{ST-SafeMDP}}}

\nocopyright

\begin{document}
\title{Safe Exploration in Markov Decision Processes with Time-Variant Safety \\using Spatio-Temporal Gaussian Process}

\author{Akifumi Wachi \\ IBM Research \\ akifumi.wachi@ibm.com
\And Hiroshi Kajino \\ IBM Research \\ kajino@jp.ibm.com
\And Asum Munawar \\ IBM Research \\ asim@jp.ibm.com}
\maketitle

\begin{abstract}
In many real-world applications (e.g., planetary exploration, robot navigation), an autonomous agent must be able to explore a space with guaranteed safety.
Most safe exploration algorithms in the field of reinforcement learning and robotics have been based on the assumption that the safety features are a priori known and time-invariant. This paper presents a learning algorithm called \algo \ for exploring Markov decision processes (MDPs) that is based on the assumption that the safety features are \textit{a priori unknown} and \textit{time-variant}. In this setting, the agent explores MDPs while constraining the probability of entering unsafe states defined by a safety function being below a threshold. The unknown and time-variant safety values are modeled using a spatio-temporal Gaussian process. However, there remains an issue that an agent may have no viable action in a shrinking true safe space. To address this issue, we formulate a problem maximizing the cumulative number of safe states in the worst case scenario with respect to future observations. The effectiveness of this approach was demonstrated in two simulation settings, including one using real lunar terrain data.
\end{abstract}

\section{Introduction}
The fundamental challenge of decision making in many real-world applications is that safety is often \textit{a priori unknown} and \textit{time-variant}.
For example, in the case of Mars rovers, the safe space is a priori unknown and is discovered by the agent or manually defined by human operators; that is, potential hazards must be discovered on-the-fly. 
Furthermore, especially in the lunar polar regions, factors related to safety such as illumination (associated with power generation and visibility) change drastically over time. For the on-board risk-sensitive and reactive decision making, the agent should not depend on the human operators mainly because of the communication delay. 
Therefore, the agent, such as a lunar rover, that autonomously explores an uncertain environment must \textit{learn} where it is safe and when a state is safe.

In previous work on safe reinforcement learning \cite{hans2008safe,garcia2015comprehensive} and safety-constrained decision-making \citep{fleming1995risk,schwarm1999chance,blackmore2010probabilistic} safety was assumed to be guaranteed by assigning unsafe states a high cost. In other words, safety was assumed to be a priori known and time-invariant.
A useful approach to dealing with a priori unknown functions is to use a Gaussian process (GP, see \citet{rasmussen2006gaussian}). A GP is well suited for safety-critical problem settings in which an action known to be safe is taken during exploration. This is because the confidence interval (i.e., variance) as well as the mean value of an a priori unknown safety function can be leveraged. Due to being based on the assumption that the function (i.e., smoothness or continuity) is regular, GP-based safe exploration algorithms have generally provided efficient solutions as well as a rigorous theoretical guarantee of safety. Previous work on safe exploration using GPs has addressed both stateless settings \citep{sui2015safe,sui2017correlational} and stateful settings \citep{turchetta2016safe,berkenkamp2017safe,wachi2018safe}. However, in both settings, safety functions were assumed to be time-invariant. In real applications, an actual safety function often changes over time. To the best of our knowledge, current applications of GP-based safe exploration algorithms are in medical care (e.g., clinical treatment) and robotics (e.g., planetary exploration). For example, in the planetary exploration, since the safety features are associated with the planetary environment, the safety function must be treated as time-variant.

When safety is considered to be time-variant, a simple approach is to use spatio-temporal GPs \cite{hartikainen2011sparse,sarkka2012infinite}. By incorporating the spatio-temporal information into the kernel function, the safety function values are predicted in both spatial and temporal directions. However, we cannot solve a critical issue if we just use spatio-temporal GPs; that is, in the time-variant environment, the true safe space may shrink, and the agent may not be able to take any viable actions. In the time-invariant case, the true safe space will not change, and the predicted safe space will expand or at least remain unchanged. In our problem setting, states that were previously identified as safe may no longer be safe. A conceptual image is shown in Figure~\ref{fig:concept}. In the rest of the paper, \textit{safe space} indicates \textit{predicted safe space}.

\vspace{-5mm}
\paragraph{Contribution} To address the above issue, we define a safe space for a time-variant safe scenario and then formulate a min-max problem in which the cumulative number of safe states in the worst case scenario with respect to future observations is maximized. Under the worst case assumption that the safety function is Lipschitz continuous, safety is guaranteed in subsequent steps while reducing the possibility that an agent gets stuck due to having no viable action it can take.

In this paper, we present a novel safety-constrained exploration algorithm called Spatio-Temporal Safe Markov decision process (\algo). At a high level, our approach models the safety function using a spatio-temporal GP model. It first predicts future safety function values and then uses the GP model to compute the safe space defined for the time-variant safety scenario. Finally, it chooses the next target sample such that the agent is motivated to explore an uncertain state that is likely to expand the safe space while preventing the current safe space from shrinking.

We theoretically analyze and show that \algo\ guarantees safety with high probability. We also demonstrate empirically that our approach safely and efficiently explores MDPs by using a time-variant safety function. Finally, we describe and present the results of two simulations, one performed using a synthetic environment and one performed using a real environment.

\section{Problem Statement}

We consider a MDP with time-variant safety characterized by a tuple $\mathcal{M} = \langle\mathcal{S}$, $\mathcal{T}$, $\mathcal{A}$, $f(\bm{s}, a)$, $g(t, \bm{s}) \rangle$, with $\mathcal{S}$ the set of states $\{\bm{s}\}$, $\mathcal{T}$ the set of times $\{\bm{t}\}$, $\mathcal{A}$ the set of actions $\{a\}$, $f: \mathcal{S} \times \mathcal{A} \rightarrow \mathcal{S}$ the (time-invariant) deterministic transition model, and $g: \mathcal{T} \times \mathcal{S} \rightarrow \mathbb{R}$ the safety function. We assume that the safety function is {\it time-variant} and {\it not known a priori}. At each time step $t$, the agent must be in a safe state and observes noise-perturbed value of $g$; hence, it observes $y_t = g(t, \bm{s}_t) + n_t$, where $n_t$ is the noise. The safety function value $g(t, \bm{s}_t)$ of state $\bm{s}_t$ must be above some threshold, $h \in \mathbb{R}$; that is,
$
g(t, \bm{s}_t) \ge h.
$
Note that we consider the safety function to depend on only the time and state. 

To make this problem tractable, we make an assumption on regularity in the form of a GP to represent the safety function; that is, we assume that similar states possess similar levels of safety at similar times.\footnote{We do not focus on environmental hazards that 1) can be expressed as a binary function (e.g., cliffs or rocks) or 2) suddenly emerge between one step and the next.} The values of $g$ for unvisited states at a future time are predicted using the GP on the basis of previous observations for visited states. 

To learn the a priori unknown safety function, the agent should \textit{explore} the state space. 
In this paper, we focus on finding a good policy for efficiently expanding the safe space in the environment with time-variant and a priori unknown safety.\footnote{Our objective differs from that of typical reinforcement learning frameworks in which the objective is to find the policy that maximizes the cumulative reward.} Hence, we formulate the problem of finding the optimal next target state as
\begin{alignat}{2}
\label{eq:objective_func}
& \quad \max_{\bm{s}_t} \left( \min_{y_ {t+1}, \cdots, y_N} \sum_{i=t}^{N}  q(i, \bm{s}_i; \bm{y}_{i-1}) \right)\\
\label{eq:constraint}
&\text{subject to} \ \  g(i, \bm{s}_i) \ge h, \quad \forall i = [t, N],
\end{alignat}
where $\bm{y}_{i-1} = [y_1, y_2, \cdots, y_{i-1}]$, $N$ is the terminal time step, and $q(i, \bm{s}_i; \bm{y}_{i-1})$ represents the number of safe states that are identified on the basis of the set of observations $\bm{y}_{i-1}$. In the rest of this paper, we denote the objective function as $J$ such that $J(t, \bm{s}_t; \bm{y}_{i-1}) = \sum_{i=t}^{N}  q(i, \bm{s}_i; \bm{y}_{i-1})$.
Intuitively, in this problem formulation, we seek to find the optimal state to visit such that the number of safe states is maximized under the assumption that the set of observations $[y_{t+1}, y_{t+2}, \cdots, y_N]$ is most unfavorable.

\begin{figure}[t]
\centering
	\includegraphics[width=60mm]{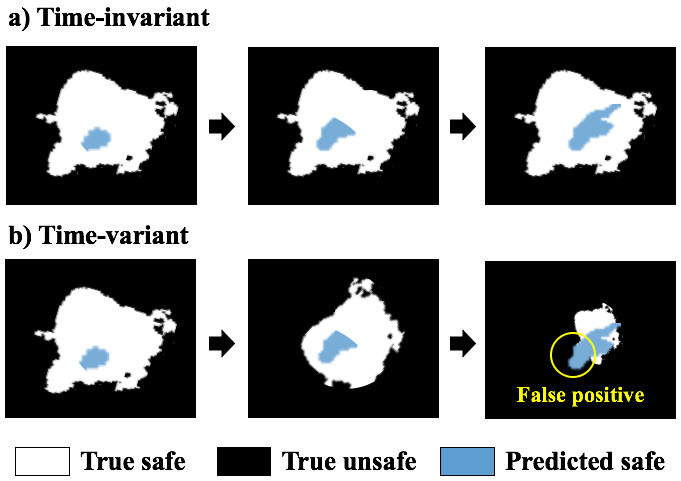}
	\caption{Conceptual image of the problem. In the time-invariant case, the true safe space (white region) does not change. However, in the time-variant case, the true safe space may shrink. Hence, without considering it, the predicted safe space (blue region) may partially overlap the true unsafe region (black region).}
	\label{fig:concept}
\end{figure}

In our algorithm, we suppose that the agent starts in a safe position. We denote by $S_0$ the set of initial safe states, which are assumed to be a priori known. This assumption is consistent with relevant previous work (e.g., \citet{sui2015safe}).

\section{Previous Work}

This work is based on several previous studies on safe exploration algorithms using GPs \cite{sui2015safe,sui2017correlational,turchetta2016safe,berkenkamp2017safe,wachi2018safe}.

Here we focus on the work by \citet{turchetta2016safe} as it is the primary basis for this paper. For better understandability, we append superscript ``$\alpha$'' to the  variables used for the explanation. \citet{turchetta2016safe} addressed the problem of safely exploring finite MDPs. They aimed to explore MDPs under the constraint of safety, assuming that the a priori unknown function was time-invariant and satisfied regularity conditions expressed using a GP prior; that is, the safety function was modeled as a GP, and specified using the mean and covariance:
$g^\alpha(\bm{s}) = \mathcal{GP}(\mu^\alpha(\bm{s}), k^\alpha(\bm{s}, \bm{s}'))$,
which is parameterized using a kernel function $k^\alpha(\bm{s}, \bm{s}')$. Observation noise was modeled as $\bm{y}^\alpha_t = g^\alpha(\bm{s}_t) +n_t$, $n_t \sim \mathcal{N}(0, \omega_t^2)$. The posterior over $g^\alpha$ was analytically calculated on the basis of $T$ measurements for states, $\{\bm{s}_1, \cdots, \bm{s}_T\}$, which is also a Gaussian distribution with mean $\mu^\alpha_T(\bm{s})$, variance $(\sigma^\alpha_T(\bm{s}))^2$, and covariance $k^\alpha_T(\bm{s}, \bm{s}')$.
They assumed that the safety function was Lipschitz continuous with Lipschitz constant $L_s$ with respect to some metric, $d_s(\bm{s}, \bm{s}')$.

\vspace{2mm}
\noindent
\textbf{Safe space in time-invariant scenario} \space\space To guarantee safety, \citet{turchetta2016safe} maintained two sets. The first set, $S^\alpha_t$, contained states
that were highly likely to satisfy the safety constraint using the GP posterior. The second
one, $\hat{S}^\alpha_t$, additionally considered the ability to reach points in $S^\alpha_t$ (i.e., reachability constraint) and the ability to safely return to the previous identified safe set, $\hat{S}^\alpha_{t-1}$ (i.e., returnability constraint). The algorithm guarantees safety by  allowing visits only to states in $\hat{S}^\alpha_t$.

Safety is evaluated on the basis of the confidence interval represented as, $Q^\alpha_t(\bm{s}) = [\mu^\alpha_{t-1}(\bm{s}) \pm \beta_t^{1/2} \sigma^\alpha_{t-1}(\bm{s})]$,
where $\beta_t$ is a positive scalar representing the required level of safety. 
Consider the intersection of $Q_t$ up to iteration $t$. It is defined as $C^\alpha_t(\bm{s}) = Q^\alpha_t(\bm{s}) \cap C^\alpha_{t-1}(\bm{s})$, $C^\alpha_0(\bm{s}) = [h, \infty],\ \forall \bm{s} \in S_0$. 
The lower and upper bounds on $C^\alpha_t(\bm{s})$ are denoted by $l^\alpha_t(\bm{s}):=\min C^\alpha_t(\bm{s})$ and $u^\alpha_t(\bm{s}):=\max C^\alpha_t(\bm{s})$, respectively. The first set, $S^\alpha_t$ is defined using $L_s$.
\begin{equation*}
S^\alpha_t = \{ \bm{s} \in \mathcal{S} \mid \exists \bm{s}' \in \hat{S}^\alpha_{t-1}:
l^\alpha_t(\bm{s}') - L_s d_s(\bm{s}, \bm{s}') \ge h \}.
\end{equation*}

Next, the reachable and returnable sets are defined. Even if a state is in $S^\alpha_t$, it may be surrounded by the unsafe states. Given a set $\mathcal{X}$, the set that is reachable from $\mathcal{X}$ in one step is
\[
R^\alpha_{\text{reach}}(\mathcal{X}) = \mathcal{X}\cup\{\bm{s} \in \mathcal{S} \mid \exists \bm{s}' \in \mathcal{X}, a \in \mathcal{A}: \bm{s}=f(\bm{s}', a)\}.
\]

An agent may be unable to move to another state due to the lack of safe actions. The set of states from which the agent can return to $\bar{\mathcal{X}}$ through other set of states $\mathcal{X}$ in one step is given by 
\[
R^\alpha_{\text{ret}}(\mathcal{X},\bar{\mathcal{X}}) = \bar{\mathcal{X}} \cup \{\bm{s} \in \mathcal{X} \mid \exists a \in\mathcal{A}: f(\bm{s}, a) \in \bar{\mathcal{X}} \}.
\]
Next, the set containing all the states that can reach $\bar{\mathcal{X}}$ through an arbitrary long path in $\mathcal{X}$ is defined as $\bar{R}^{\alpha}_{\text{ret}}(\mathcal{X},\bar{\mathcal{X}})$.

Finally, the set of safe states that satisfy the reachability and returnability constraints, $\hat{S}_t$, is defined as
\begin{equation}
\label{eq:safe_set_true_turchetta}
\hat{S}^\alpha_t = \{ \bm{s} \in S^\alpha_t \mid \bm{s} \in R^\alpha_\text{reach}(\hat{S}^\alpha_{t-1}) \cap \bar{R}^{\alpha}_{\text{ret}}(S^\alpha_t, \hat{S}^\alpha_{t-1}) \}.
\end{equation}

\vspace{2mm}
\noindent
\textbf{Sampling criteria} \space\space \citet{turchetta2016safe} defined expanders, $\Xi_t^\alpha$, that might increase the number of states that are identified as safe. Then the state with the highest variance in $\Xi_t^\alpha$ was selected as the next target sample; that is, the next sample was
$
\bm{s}_t = \argmax_{\bm{s} \in \Xi_t^\alpha} w_t^\alpha(\bm{s}),
$
where $w_t^\alpha(\bm{s}) = u_t^\alpha(\bm{s}) - l_t^\alpha(\bm{s})$.

\section{Characterization of Spatio-Temporal Safety}
\label{sec:gp}

Characterization of safety is dealt with differently here than by \citet{turchetta2016safe} in two aspects. First, we deal with time-variant safety, meaning that we must capture the similarity between times as well as that between states, which motivated us to model the environment using spatio-temporal GPs. Second, suppose that an agent considers the set of states to which a return action is possible at time $t$. If the environment is time-invariant, the agent only has to ensure that it can return to the states in $\hat{S}_{t-1}$ because the states identified as safe at time $t-1$ are surely still safe at time $t$. However, in a time-variant environment, the true safe region may shrink; that, in turn, the predicted safe region may also shrink. 
This means that the agent must consider the states to which return is possible at $t+1$, i.e., the states in $\hat{S}_{t+1}$, since it is at time $t+1$ that a return action is taken. However, this requirement is problematic: $\hat{S}_{t}$ is defined by $\hat{S}_{t+1}$. Hence, we approximate $\hat{S}_{t+1}$ by using a conservative set.

\subsection{Spatio-Temporal Gaussian Processes} 

\begin{figure}[t]
	\centering
	\includegraphics[width=70mm]{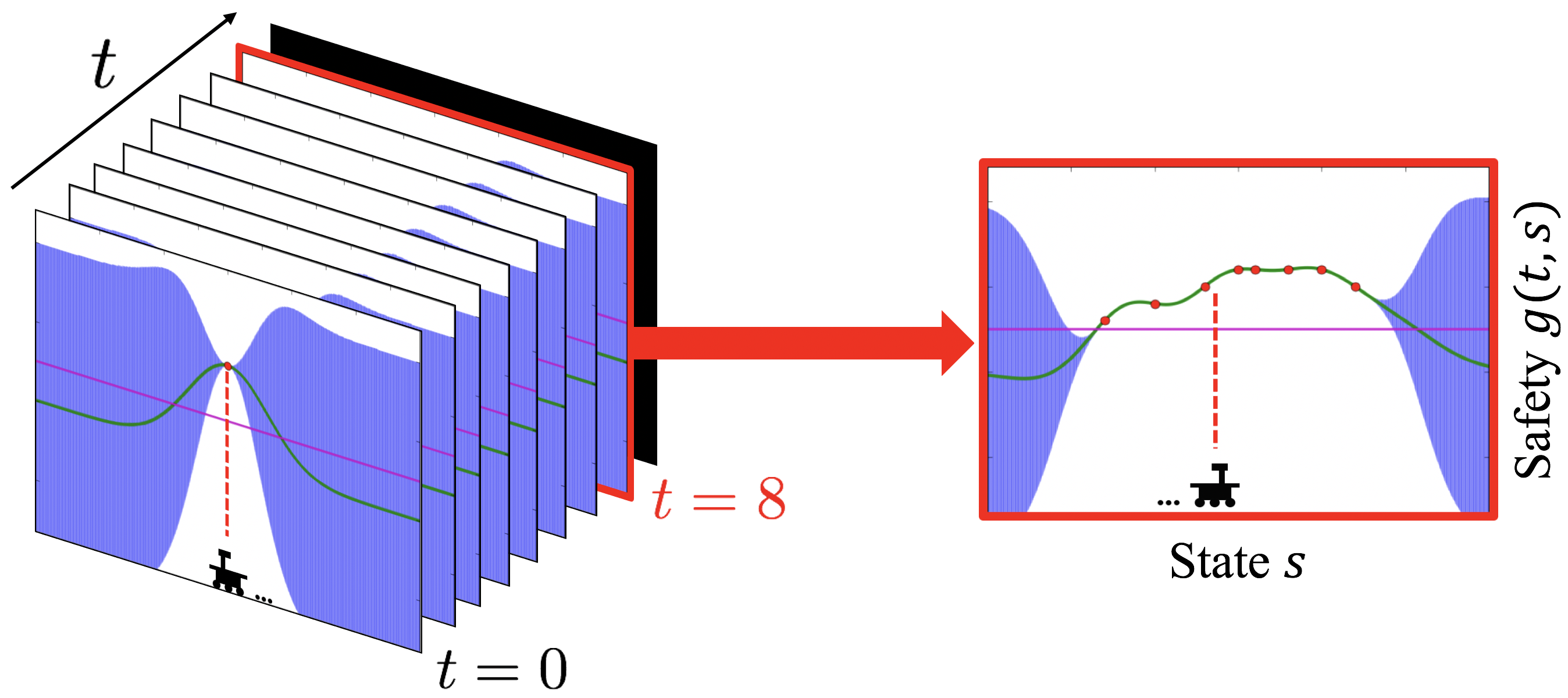}
    \caption{Conceptual image of characterization of safety. Green curve and blue band represent mean and confidence interval, respectively. Magenta line represents safety threshold $h$. Forefront figure on left is for $t=1$, and figure on right is for $t=8$. A black rectangle represents the time step with no observation.}
\label{fig:GP}
\end{figure}

As mentioned previously, we make an assumption on regularity in the form of a GP to represent the safety function, $g(\bm{t}, \bm{s})$. The GP-based predictions have uncertainty, which is represented by Gaussian distributions. Roughly speaking, an unvisited state is considered safe if the safety function value is almost certainly above a predefined threshold (i.e., with a greater probability than a predefined level). 
We assume that $\mathcal{S}$ and $\mathcal{T}$ are endowed with a positive semidefinite kernel function and that the safety function has bounded norm in the associated Reproducing Kernel Hilbert Space (RKHS). Kernel function $k$ is used to capture the similarity between states and that between times; it formalizes the assumption of safety function regularity. These concepts are shown in Figure \ref{fig:GP}. Incorporating the spatio-temporal information into the kernel function enables attaching larger weights to more recent observations. The safety function can thus be modeled as
\[
g(\bm{t}, \bm{s}) = \mathcal{GP}(\mu(\bm{t}, \bm{s}), k(\{\bm{s}, \bm{t}\}, \{\bm{s}', \bm{t}'\})).
\]
We define a new variable, $\eta=\{ \bm{t}, \bm{s} \}$, and denote the set of $\bm{\eta}$ as $\mathcal{H}$. A GP is fully specified by its mean, $\mu(\bm{\eta})$, and covariance, $k(\bm{\eta}, \bm{\eta}')$. Without loss of generality, let $\mu(\bm{\eta}) = 0$ for all $\bm{s} \in \mathcal{S}$ and $\bm{t} \in \mathcal{T}$. We model observation noise as $y = g(\bm{\eta}) + n_t$, where $n_t \sim \mathcal{N}(0, \omega_t^2)$. The posterior  over $g(\bm{\eta})$ is computed analytically on the basis of $T$ measurements for times and states $\bm{\mathcal{H}}_T = \{\{1, \bm{s}_1\}, \cdots, \{T, \bm{s}_T\}\}$ with measurements, 
$\bm{y}_T = [g(\bm{\eta}_1)+n_1, \cdots, g(\bm{\eta}_T)+n_T]^\top$.
The posterior mean, variance, and covariance are given as:
\begin{alignat}{6}
\label{eq: GP_model}
\nonumber
\mu_T(\bm{\eta}) =  &\ \bm{k}_T(\bm{\eta})^\top(\bm{K}_T+\omega_t^2\bm{I})^{-1}\bm{y}_T \\
\sigma^2_T(\bm{\eta}) = &\ k_T(\bm{\eta}, \bm{\eta}) \\
\nonumber
k_T(\bm{\eta}, \bm{\eta}') = &\ k(\bm{\eta}, \bm{\eta}') - \bm{k}_T(\bm{\eta})^\top(\bm{K}_T+\omega_t^2\bm{I})^{-1}\bm{k}_T(\bm{\eta}'),
\end{alignat}
where $\bm{k}_T(\bm{\eta}) = [k(\bm{\eta}_1,\bm{\eta}) \cdots, k(\bm{\eta}_T, \bm{\eta})]^\top$ and $\bm{K}_T$ is the positive definite kernel matrix, $[k(\bm{\eta}, \bm{\eta}')]_{\bm{\eta}, \bm{\eta}' \in \bm{\mathcal{H}}_T}$. For more details on spatio-temporal GPs, we refer the reader to \citet{soh2012online} or \citet{senanayake2016predicting}. 
We further assume that the safety function $g$ is Lipschitz continuous, with Lipschitz constant $L_s$ (for states) and $L_t$ (for time), with respect to some metric $d_s(\cdot,\cdot)$ on $\mathcal{S}$ and $d_t(\cdot,\cdot)$ on $\mathcal{T}$. This assumption is naturally satisfied using common kernel functions \citep{ghosal2006posterior}. For simplicity, we define $L$ such that $L(\bm{s},\bm{s}',\bm{t}, \bm{t}') = L_s d_s(\bm{s},\bm{s}') + L_t d_t(\bm{t}, \bm{t}')$.
For all $t$, the time step between $t-1$ and $t$ is assumed to be consistent; that is, $d_t(t',t) = |t-t'|\cdot d_t(t-1,t)$.

The kernels used in our model are composite kernels, either product kernels, $(k_s \otimes k_t)(\{\bm{s},\bm{t}\},\{\bm{s}', \bm{t}'\}) = k_s(\bm{s},\bm{s}')k_t(\bm{t},\bm{t}')$ or additive combinations, $(k_s \oplus k_t)(\{\bm{s},\bm{t}\},\{\bm{s}', \bm{t}'\}) = k_s(\bm{s},\bm{s}') + k_t(\bm{t},\bm{t}')$, where $k_s$ is the kernel for space, and $k_t$ is the one for time. In our simulation, we used a composite kernel such that
\begin{equation}
\label{eq:st_kernel}
k = (k_s \oplus k_t) \oplus (\hat{k}_s \otimes \hat{k}_t),
\end{equation}
where $\hat{k}_s$ and $\hat{k}_t$ are kernel functions for space and time, respectively. For more details on composite kernels, see \citet{krause2011contextual} or \citet{duvenaud2013structure}.

\subsection{Safe Space in Time-Variant Scenario}
\label{sec:safety}

We use the extended definition of the safe space as given by \citet{turchetta2016safe}. Likewise, to guarantee safety (with high probability), we maintain two sets. The first set is $S_t$, which contains all states that satisfy the safety constraint (\ref{eq:constraint}) with high probability. The second one, $\hat{S}_t$, includes reachability and returnability constraints in addition to the safety constraint (see Figure~\ref{fig:concept_constraints}). 

First, we formally define $S_t$, the set of states that satisfy the safety constraint. A GP model enables an agent to judge safety by providing it a confidence interval having the form $Q_t(\bm{\eta}) = [\mu_{t-1}(\bm{\eta}) \pm \beta_t^{1/2} \sigma_{t-1}(\bm{\eta})]$, where $\beta_t$ is a positive scalar specifying the required level of safety. In other words, $\beta$ inherently specifies the probability of violating the safety constraint. For more details on $\beta$, see \citet{srinivas2009gaussian}. We then consider the intersection of $Q_t$ up to iteration $t$; it is recursively defined as $C_t(\bm{\eta}) = Q_t(\bm{\eta}) \cap C_{t-1}(\bm{\eta})$, $C_0(\bm{\eta}) = [h, \infty],\ \forall \bm{s} \in S_0$. We denote the lower and upper bounds on $C_t(\bm{\eta})$ by $l_t(\bm{\eta}):=\min C_t(\bm{\eta})$ and $u_t(\bm{\eta}):=\max C_t(\bm{\eta})$, respectively. The first set, $S_t$, is defined using the Lipschitz constant $L$ (i.e., $L_s$ and $L_t$):
\begin{equation*}
\begin{split}
S_t = 
\{ \bm{s} \in \mathcal{S} \mid \ \exists \bm{s}' \in \hat{S}_{t-1}: &\ l_t(\{t-1,\bm{s}'\}) \\
& - L(\bm{s},\bm{s}',t, t-1) \ge h \}. \\
\end{split}
\end{equation*}

\noindent
\textbf{Reachability and returnability} \space\space We define the set of states that are reachable and those to which return is possible. Even if a state is in $S_t$, it may not be reachable without visiting unsafe states first. For the reachable set, we use the same definition as \citet{turchetta2016safe}. Given $\mathcal{X}$, the set of states reachable from $\mathcal{X}$ in one step is
\[
R_{\text{reach}}(\mathcal{X}) = R^\alpha_{\text{reach}}(\mathcal{X}).
\]

Next, an agent may be ``trapped'' in a state due to the lack of safe actions. As for the returnable set, we use different definition from $R^\alpha_{\text{ret}}(\mathcal{X})$. Given a set $\mathcal{X}$, the set that is returnable to $\mathcal{X}$ with one step is given as 
\[
R_{\text{ret}}(\mathcal{X}) = \mathcal{X} \cup \{\bm{s}\in\mathcal{S} \mid \exists a \in\mathcal{A}: f(\bm{s}, a) \in \mathcal{X} \}.
\]

Finally, the set of safe states is defined. Safe states must 1) satisfy the safety constraint, 2) be reachable from $\hat{S}_{t-1}$, and 3) be returnable to $\hat{S}_{t+1}$; that is, $\hat{S}_t$ is denoted as
\begin{equation}
\label{eq:safe_set_true}
\hat{S}_t = \{ \bm{s} \in S_t \mid \bm{s} \in R_\text{reach}(\hat{S}_{t-1}) \cap R_\text{ret}(\hat{S}_{t+1})\}.
\end{equation}
Observe that the returnable set is written as $R^\text{ret}(\hat{S}_{t+1})$. In a time-variant environment, an agent must be able to visit $\hat{S}_{t+1}$ since $\hat{S}_{t-1}$ and $\hat{S}_{t}$ may no longer be safe. However, defining $\hat{S}_t$ using $\hat{S}_{t+1}$ is obviously problematic.

\begin{figure}[t]
\centering
\subfigure[]{\includegraphics[width=3.3cm]{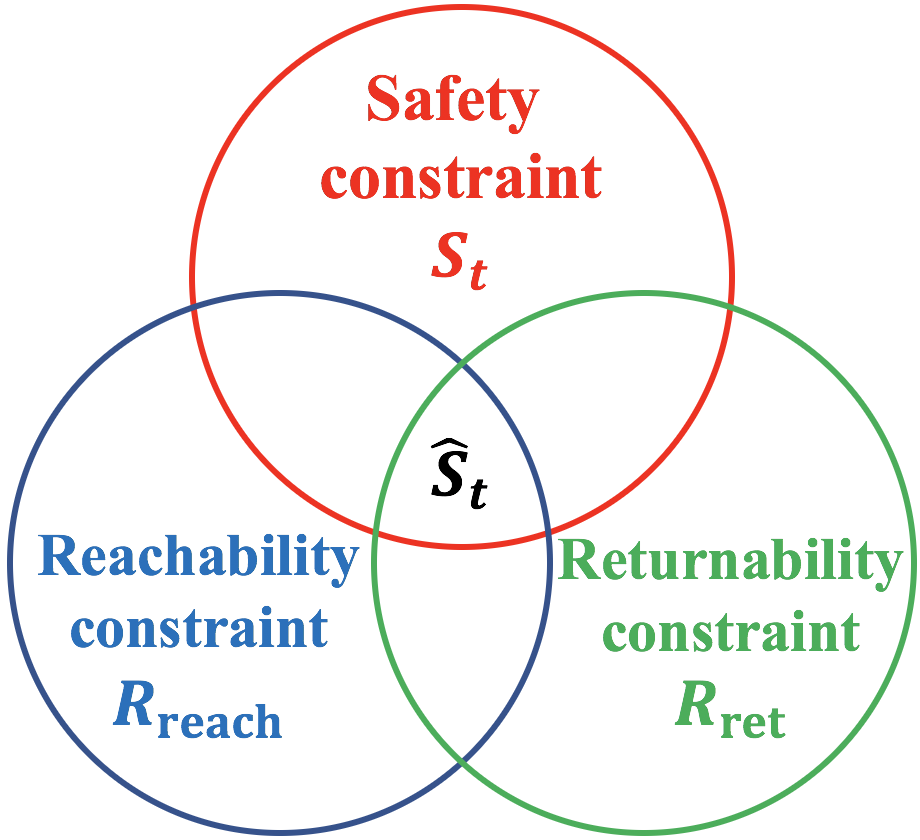}
\label{fig:concept_constraints}}
\hspace{1mm}
\subfigure[]{\includegraphics[width=4.6cm]{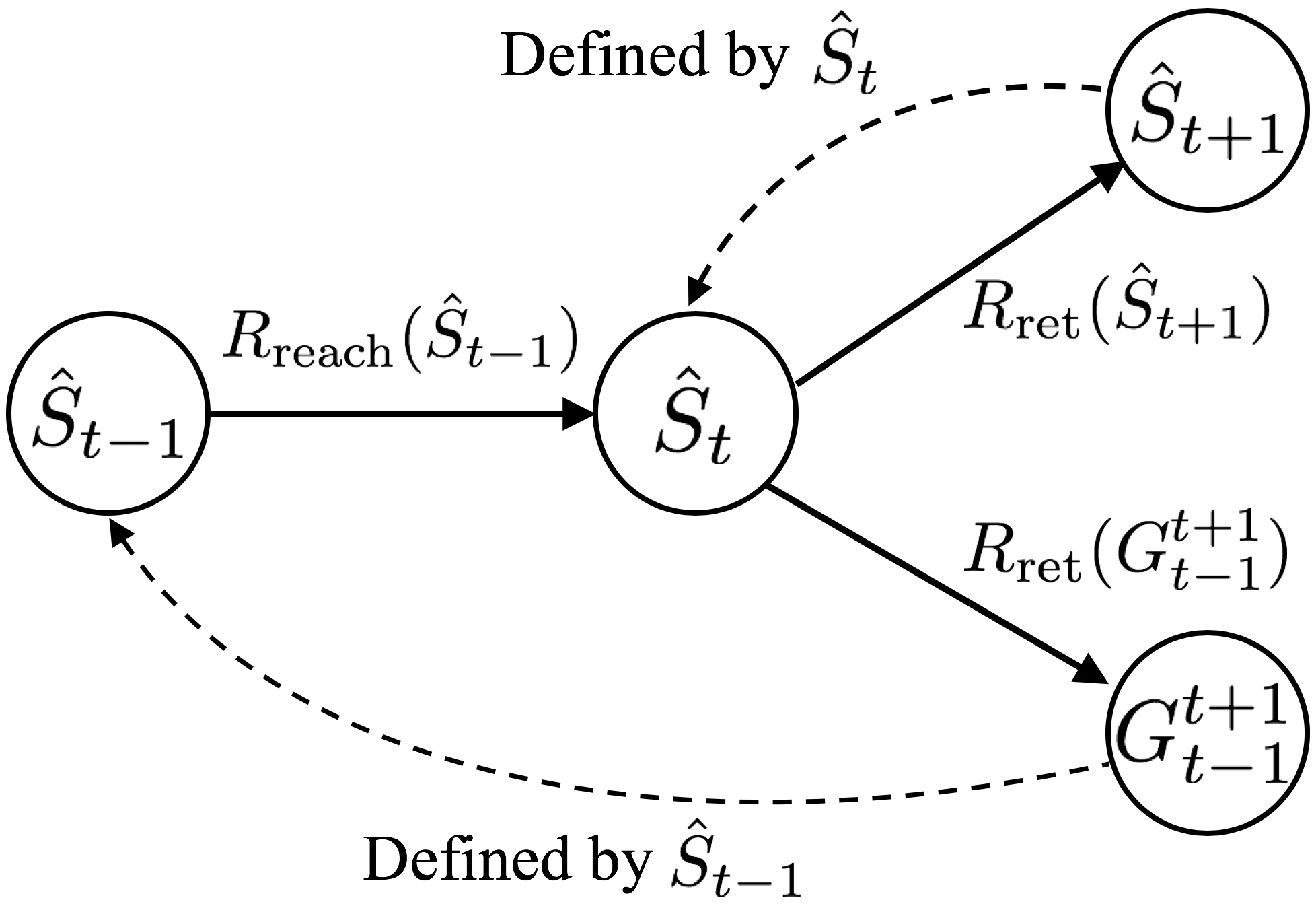}
\label{fig:concept_sets}}
\caption{(a) Safe space is characterized by three constraints. (b) Conceptual image of the relations between each set. It is problematic to define $\hat{S}_{t}$ using $\hat{S}_{t+1}$; hence, we use $G_{t-1}^{t+1}$ instead to define $\hat{S}_t$.}
\end{figure}

\vspace{2mm}
\noindent
\textbf{Approximation of $\bm{\hat{S}_{t+1}}$} \space\space We cannot directly solve (\ref{eq:safe_set_true}), but nonetheless we have to consider returnability in order to avoid the situation in which an agent is unable to take any action at time $t$. Suppose we have a set of observations, $\bm{y}_{t-1}$. For any pair of time steps, $\tau$ and $\tau' (< \tau)$, we define a conservative set of safe states,
\begin{equation*}
\begin{split}
G_{\tau'}^\tau(\bm{y}_{t-1}) =
\ \{ \bm{s} \in \mathcal{S} \mid &\ \exists \bm{s}' \in \hat{S}_{\tau'}:  \\
&\ l_t(\{\tau',\bm{s}'\}) - L(\bm{s},\bm{s}', \tau, \tau') \ge h \}. 
\end{split}
\end{equation*}
Intuitively, $G$ is an extended definition of $S$ in that $G$ takes into account the satisfaction of the safety constraint several steps into the future. Hence, in (\ref{eq:safe_set_true}), we make an approximation and replace $\hat{S}_{t+1}$ by $G_{t-1}^{t+1}$; that is,
\begin{equation*}
\label{eq:safe_set_approx}
\hat{S}_t \approx \{ \bm{s} \in S_t \mid \bm{s} \in R_\text{reach}(\hat{S}_{t-1}) \cap R_\text{ret}(G_{t-1}^{t+1}(\bm{y}_{t-1}))\}.
\end{equation*}
Under the assumption of Lipschitz continuity, $G_{t-1}^{t+1}$ can be identified as safe on the basis of observations until time $t-1$, and taking into account $R_\text{ret}(G_{t-1}^{t+1})$ enables us to avoid the situation in which an agent cannot move to another state. A conceptual image is shown in Figure~\ref{fig:concept_sets}. 
%

\section{Approach}
\label{sec:alg}

\begin{algorithm}[t]
\caption{\textbf{\space \algo}}
\label{algorithm1}
\begin{small}
\begin{algorithmic}[1]
\STATE{\textbf{Inputs:} MDP $\mathcal{M} = \langle \mathcal{S}, \mathcal{T}, \mathcal{A}, f(\bm{s}, a), g(t, \bm{s}) \rangle$,  Safety threshold $h$, Lipschitz constraints $L_s$ and $L_t$, Initial safe set $S_0$}
\STATE{Construct spatio-temporal kernel}
\STATE $C_0(s) \leftarrow [h, \infty), \quad \forall s \in S_0$ 
\FOR {$i = 1, 2, \cdots$}
\STATE $\hat{S}_t \leftarrow \{ \bm{s} \in S_t \mid \bm{s} \in R_\text{reach}(\hat{S}_{t-1}) \cap R_\text{ret}(G_{t-1}^{t+1})\}$
\STATE $\Xi_t \leftarrow \{ s \in \hat{S}_t \mid \xi_t(s) > 0 \}$
\STATE $\Psi_t = \{ \bm{s} \in \mathcal{S} \mid \exists a \in \mathcal{A}: \bm{s} = f(\bm{s}_{t-1}, a)\}$
\IF {$\Xi_t \cap \Psi_t = \emptyset$}
\STATE $\bm{s}_t \leftarrow \argmax_{\bm{s} \in \Xi_t \cap \Psi_t} \left( \mu_t(\{t,\bm{s}\}) + p \cdot w_t(\{t,\bm{s}\}) \right)$
\ELSE
\STATE $\bm{s}_t \leftarrow \argmax_{\bm{s} \in \hat{S}_t \cap \Psi_t} \left( \mu_t(\{t,\bm{s}\}) + p \cdot w_t(\{t,\bm{s}\}) \right)$
\ENDIF
\STATE{$y_t \leftarrow g(t, \bm{s}_t) + n_t$}
\STATE Update GP model using $\bm{s}_{t}$ and $\bm{y}_{t}$
\ENDFOR
\end{algorithmic}
\end{small}
\end{algorithm}

We start by giving a high level overview of our \algo\ algorithm (see Algorithm~\ref{algorithm1}). At each iteration, the agent updates the safe space (Line~5) by using the GP model. Given the safe set, $\hat{S}_t$, the agent calculates expanders that might enhance the number of states identified as safe (Line~6). Then, a state with the maximum weighted summation of the mean and confidence interval of the safety function is chosen as the next target sample (Lines~8-12). Finally, the agent observes the safety function value of the current state and updates the GP model (Lines~13 and 14).

Our algorithm maximizes the lower bound of the cumulative number of safe states in the next multiple time steps, as shown by the lemmas and theorem below. Hereinafter, we explicitly denote the underlying observations as independent variables. Recall that our objective here is to find a good policy for efficiently expanding the safe space in the environment with time-variant and a priori unknown safety, as expressed by~(\ref{eq:objective_func}). Given that $q(i, \bm{s}_i; \bm{y}_{i-1})$ corresponds to $|\hat{S}_i(\bm{y}_{i-1})|$, the following lemmas hold.

\begin{lemma}
\label{theorem_observation}
The following inequality holds:
\begin{alignat*}{2}
J(t, \bm{s}_t; \bm{y}_{i-1})
\ge |\hat{S}_t(\bm{y}_{t-1})| + \sum_{i=t+1}^{N}  |\hat{S}_i(\bm{y}_{t})|.
\end{alignat*}
\end{lemma}

\begin{lemma}
\label{lem:hat_S_recursive2}
For all $i = [t, N]$, the following inequality holds:
\begin{alignat*}{2}
|\hat{S}_i(\bm{y}_t)| 
\ge &\ |\hat{S}_{t}(\bm{y}_t) \cap G_{t}^{i+1}(\bm{y}_t)|. 
\end{alignat*}
\end{lemma}

These lemmas come from our desire to express the objective function using only $\bm{y}_{t-1}$ and $\bm{y}_t$. Finally, we obtain the following theorem on the lower bound of the objective function using Lemmas~\ref{theorem_observation} and~\ref{lem:hat_S_recursive2}.
\begin{theorem}
\label{theorem:lower_bound_J}
The following equality holds:
\begin{alignat*}{2}
\min_{y_ {t+1}, \cdots, y_N} J(t, \bm{s}_t; \bm{y}_{t-1}) = |\hat{S}_t(\bm{y}_{t-1})| + \sum_{i=t+1}^{N}  M_i(\bm{y}_t),
\label{eq:V_M}
\end{alignat*}
where $M_i$ is defined as
$M_i := |\hat{S}_{t} \cap G_{t}^{i+1}|$. 
\end{theorem}
Note that $J$ has a lower bound characterized by only $\bm{y}_{t-1}$ and $\bm{y}_t$. Proofs of Lemma~\ref{theorem_observation}, \ref{lem:hat_S_recursive2} and Theorem~\ref{theorem:lower_bound_J} are given in the supplemental material. Now, the problem to solve is, \textit{what is the optimal state to obtain the observation, $y_t$?}.

Let us carefully observe $M_i(\bm{y}_t)$. Once we transform $\sum_{i=t+1}^{N} M_i$ into $M_{t+1} \cdot \sum_{i=t+1}^{N} M_i/M_{t+1}$, we get a clearer view for optimization. First, of the two components of $M_{t+1}$, $\hat{S}_t$ is deterministic while $G_t^{t+2}$ is not unknown at this moment (because an agent has not yet obtained the observation at time $t$, $y_t$). As for $M_i, \forall i = [t+2, N]$; since they depend on $M_{i-1}$, it is difficult to directly maximize $M_i$. However, $M_i/M_{t+1}$ represents the ratio of the number of states within $\hat{S}_t$ that satisfy $l_t-L(\bm{s},\bm{s}',t, i) \ge h$ with respect to one that satisfies $l_t-L(\bm{s},\bm{s}',t,t+2) \ge h$. Therefore, to maximize $J$, we choose as the next target sample one that we can expect its observation to 1) expand $M_{t+1}(\bm{y}_t)$ and 2) enhance $M_i(\bm{y}_t)/M_{t+1}(\bm{y}_t)$ (i.e., choose a state whose safety function value is expected to be high).

As a means to achieve the first objective, we utilize \textit{expanders} (explained shortly) to enhance $M_{t+1}(\bm{y}_t)$ when an agent observes the safety function value. For the second one, we choose as the next target sample the one with the maximum value of the weighted summation of the mean and the width of the confidence interval.

\vspace{2mm}
\noindent
\textbf{Expander} \space\space As mentioned above, to enhance $M_{t+1}(\bm{y}_t)$, we utilize \textit{expanders} as in \citet{sui2015safe} or \citet{turchetta2016safe}. We use the uncertainty estimate in the GP to define an optimistic set of expanders,
\[
\Xi_t = \{ \bm{s} \in \hat{S}_t \mid \xi_t(\bm{s}) > 0 \},
\]
where $\xi$ is defined as
\begin{equation*}
\xi_t(\bm{s}) = | \{  \bm{s}' \in \mathcal{S} \setminus S_t \mid u_{t}(\{t, \bm{s}\}) - L(\bm{s}, \bm{s}', t, t+2) \ge h\}|.
\end{equation*}
The function $\xi_t(\bm{s})$ is positive whenever an optimistic measurement at $\bm{s}$ (i.e., the upper confidence bound) enables determination that a previously unsafe state satisfies the safety constraint. Intuitively, sampling $\bm{s}$ might lead to the expansion of $G_t^{t+2}$ (i.e., $M_{t+1}$). 

\begin{table*}[t]
\begin{center}
	\begin{tabular}{cccccc}
	\hline\hline
		& Normalized RMSE & Failure & Accuracy & Precision & Recall \\ \hline
		\algo & 1.00 $\pm$ 0.00 & \textbf{0} & \textbf{0.56 $\pm$ 0.12} & \textbf{1.00 $\pm$ 0.00} & \textbf{0.43 $\pm$ 0.10} \\ 
		Random & 1.98 $\pm$ 0.33 & 100 & $-$ & $-$ & $-$ \\ 
	    Unsafe & 0.72 $\pm$ 0.14 & 100 & $-$ & $-$ & $-$ \\ 
	    Ignore time-variance & 1.43 $\pm$ 0.29 & 36 & 0.34 $\pm$ 0.15 & 0.78 $\pm$ 0.09 & 0.28 $\pm$ 0.13\\ 
       No cross-covariance  & 1.50 $\pm$ 0.22 & 25 & 0.43 $\pm$ 0.14 & 0.85 $\pm$ 0.13 & 0.36 $\pm$ 0.13 \\
	\hline\hline
	\end{tabular}
    \end{center}
    \vspace{-1mm}
    \caption{\label{tab:simpleGridWorld1}Simulation results for simple grid world (100 times Monte-Carlo simulation). \algo\ achieved higher accuracy, precision, and recall without any failure.}
    \vspace{-1mm}
\end{table*}

\vspace{2mm}
\noindent
\textbf{Sampling criteria} \space \space We assume that an agent observes the safety function value at every visited state.\footnote{The algorithm proposed by \citet{turchetta2016safe} selects as the next target sample the state with the highest variance under the assumption that an agent does not observe the unknown function value on the way to the target sample. This assumption provides a rich theory for exploring the full region of safely reachable states. However, we consider it is more reasonable to assume that the unknown function value can be observed at every visited state.} This assumption is particularly significant in dealing with an environment in which the safety function value may change while the agent moves to the next target sample. Hence, we now consider the motion of the agent for one time step. We first define the set of candidates to be the next visited state as
$\Psi_t = \{ \bm{s} \in \mathcal{S} \mid \exists a \in \mathcal{A}: \bm{s} = f(\bm{s}_{t-1}, a)\}.$

Our goal is to maximize the number of safe states while guaranteeing safety. We do this by using GP posterior uncertainty regarding the states in $\Xi_t \cap \Psi_t$. At each iteration $t$, we select as the next target sample the state in $\Xi_t \cap \Psi_t$ with the maximum weighted summation of the mean and the width of the confidence interval: 
\begin{equation}
\bm{s}_t = \argmax_{\bm{s} \in \Xi_t \cap \Psi_t} \left( \mu_t(\{t,\bm{s}\}) + p \cdot w_t(\{t,\bm{s}\}) \right),
\end{equation}
where $w_t(\bm{\eta}) = u_t(\bm{\eta}) -l_t(\bm{\eta})$, and $p$ is a weighted coefficient. If $\Xi_t \cap \Psi_t$ is empty, we sample a new state within $\hat{S}_t \cap \Psi_t$. This sampling criteria is justified because we can gain the most information from the state for which we are the most uncertain, and states in $\Xi_t \cap \Psi_t$ can enlarge the safe set. Also, all the states in $\Xi_t \cap \Psi_t$ are guaranteed to be safe. In addition, by additionally considering the mean value, $\mu$, the agent is encouraged to visit a state that is likely to be still safe in the next multiple time steps.

\section{Theoretical Results}
\label{sec:theo}

We use the correctness of the confidence intervals, $C_t(\bm{\eta})$, as the metric for the safety guarantee. That is, the true safety function value, $g(\bm{\eta})$, must be within $C_t(\bm{\eta})$. We adjust the conservativeness by tuning scaling factor $\beta_t$. Though this paper addresses the planning problem, like \citet{srinivas2009gaussian}, who studied appropriate selection of $\beta_t$ for the multi-armed bandit problem (i.e., the sampling problem), we use
\begin{equation}
\label{eq:beta}
\beta_t = 2B + 300 \gamma_t \log^3(t/\delta),
\end{equation}
where $B$ is the bound on the RKHS norm of the safety function, $\gamma_t$ is the maximum mutual information that can be gained about $g(\cdot)$ from $t$ noisy observations (i.e., $\gamma_t = \max_{|A|\le t} I (g; \bm{y}_A )$), and $\delta$ is the probability of visiting unsafe states. For more details on the bounds on $\gamma_t$, we refer the reader to \citet{srinivas2009gaussian}. As for the spatio-temporal kernel function, the existence of an upper bound on the information gain is ensured by the following lemma. 
\begin{lemma}
\label{upper_IG}
Let $\bm{k}_s$ and $\bm{k}_t$ be a finite number of kernel functions on $\mathcal{S}$ and $\mathcal{T}$, respectively. We denote an arbitrary composite kernel resulting from $\bm{k}_s$ and $\bm{k}_t$ using product kernel operators (i.e., $\otimes$) or additive combination (i.e., $\oplus$) by $k_F$. Then $k_F$ has a bound of information gain.
\end{lemma}

Since the GP-based prediction is stochastic, being in $S_t$ means the constraint on safety is fulfilled with high probability. Given that $\beta$ is chosen as in (\ref{eq:beta}), we now present a theorem on the safety guarantee.

\begin{theorem}
\label{theorem:safe_guarantee}
Assume that $g$ is L-Lipschitz continuous and satisfies $\|g\|^2_k \le B$ and that noise $n_t$ is $\sigma$-sub-Gaussian. Also assume that $S_0 \ne \emptyset$ and that $g(s) \ge h$ for all $s \in S_0$. If $\beta_t$ is chosen as in (\ref{eq:beta}) and the next state is sampled within $\Xi_t \cap \Psi_t$ or $\hat{S}_t \cap \Psi_t$ (i.e., the subsets of $S_t$), the probability of entering a safe state is guaranteed to be at least $1-\delta$.  
\end{theorem} 
Proofs of Lemma~\ref{upper_IG} and Theorem~\ref{theorem:safe_guarantee} are given in the supplemental material.

\begin{table*}[t]
\begin{center}
	\begin{tabular}{cccccc}
	\hline\hline
		& Normalized RMSE & Number of unsafe actions & Accuracy & Precision & Recall  \\ \hline
		\algo & 1.00 & \textbf{0} & \textbf{0.51} & \textbf{1.00} & \textbf{0.24}\\ 
		Random & 2.34  & 47 & $-$ & $-$ & $-$ \\ 
	    Unsafe & 0.67 & 55 & $-$ & $-$ & $-$ \\ 
	    Ignore time-variance & 1.23 & 17 & 0.45 & 0.72 & 0.21\\
	    No cross-covariance  & 1.35 & 14 & 0.48 & 0.82 & 0.22 \\
	\hline\hline
	\end{tabular}
    \end{center}
    \vspace{-2mm}
    \caption{\label{tab:lunar_terrain}Results for simulated lunar surface exploration. \algo\ outperformed four baselines in terms of safety, accuracy, precision, and recall.}
    \vspace{-1.5mm}
\end{table*}

\section{Experiments}
\label{sec:experiments}

We evaluated our approach using two problem settings. One was a randomly generated environment, and the other was an environment created from real lunar terrain data. We used a simulated lunar environment because the lunar polar regions have been attracting attention as potential sites for a moon base although the illumination conditions can change drastically over time in those regions. For surface explorers relying on solar power, it is critical to traverse places that are illuminated. That is why we simulated lunar exploration rather than Mars surface exploration as was done in related work \cite{moldovan2012safe,wachi2018safe}.

\subsection{Randomly Generated Environment}
\label{sec:MC}

For the first problem setting, we defined a 20 $\times$ 20 square grid  and randomly generated a safety function value for each state. For each state (except the boundary ones), the agent could take one of five actions: {\it stay}, {\it go up}, {\it go down}, {\it go left}, and {\it go right}. We randomly generated an environment with time-variant safety as follows. First, we defined a radial basis function (RBF) kernel with two input dimensions that had a variance of 1 and length-scales of 2. Then, we generated a multivariate normal function such that the mean was zero and the covariance matrix was one of the above RBF. A two-dimensional safety function was then created and stacked with respect to time such that $g(t+1, \bm{s}) = g(t, \bm{s}) + L_t \cdot \varphi \cdot g(1, \bm{s}), \forall \bm{s} \in \mathcal{S}$, where $\varphi$ is a sample from the uniform distribution, $[-1, 1]$. By stacking the safety function with respect to time, we obtained a randomly generated environment with a time-variant safety function. 

The agent predicts the safety function by using the GP and a kernel defined as (\ref{eq:st_kernel}). RBFs were used for the four kernels, $k_s$, $k_t$, $\hat{k}_s$, and $\hat{k}_t$ in our model. Their length-scales were 2.0, 1.5, 4.0, and 10, and the prior variances for each were 1.0, 1.0, 0.5, and 0.5, respectively. The simulation settings were $\beta_t=2\, \forall t \ge 0$, $p = 3$, $L_s = 0.1$, and $L_t = 0.1$.

\vspace{-5mm}
\paragraph{Baselines} We compare the results for our approach with those of four baselines. The first baseline (\textbf{Random}) uses a random exploration strategy in which an agent randomly chooses the next action. The second baseline (\textbf{Unsafe}) uses pure maximization of information gain in which an agent greedily gathers information without any constraint on safety. In the third baseline (\textbf{Ignore time-variance}), an agent assumes that the environment is time-invariant; that is, the agent updates the GP model using only $k_s$. Last, we consider a baseline (\textbf{No cross-covariance}) in which the agent predicts the safety function value without using the cross-covariance term; that is, the agent uses only $k_s$ and $k_t$. We performed Monte-Carlo simulation using 100 samples of randomly generated safety values.
 
\vspace{-5mm}
\paragraph{Metrics} We used five metrics for the comparison: root mean square error (RMSE) between the true and predicted safety function values (normalized w.r.t. \algo), number of failures, accuracy, precision, and recall. Accuracy, precision, and recall are defined as $\frac{\text{TP + TN}}{\text{TP + TN + FP + FN}}$, $\frac{\text{TP}}{\text{TP + FP}}$, and $\frac{\text{TP}}{\text{TP + FN}}$, respectively, in which ``positive'' means safe. Note that ``FP'' may result in a catastrophic failure, so keeping precision high is quite important. Precision is a measure of the safety guarantee, and recall is a measure of the efficiency of the expansion. 

As shown in Table~\ref{tab:simpleGridWorld1}, our approach outperformed the baselines in terms of accuracy, precision, recall, and number of failures. RMSE was better for unsafe exploration than for \algo\ because the agent visited any state with the highest variance without any consideration of safety. Note that, in order to make the total number of actions identical in all cases, the agent was allowed to explore even after failure.

\subsection{Simulated Moon Surface Exploration}

For the second problem setting, we used a lunar region about one kilometer square centered on 88.664$^\circ$~S and 68.398$^\circ$~W. We used real lunar data (\url{https://goo.gl/6eXmQP}) managed by NASA/GSFC/ASU. The resolution was 10 m, and the square grid used was 100 $\times$ 100. The simulated environment was created by first obtaining five images of the target region (Three out of the five images are shown in Figure~\ref{fig:nac_pic}) and then linearly interpolating them so that the duration between images was $\sim \!$ 1.8 hours. The overall mission period was assumed to be $\sim \!$ 15 days, so the number of images was~200. We assumed that a rover was powered by solar arrays. Safety function $g$ was defined as the luminance value. Each pixel in each image had luminance value information, and the values were normalized such that the maximum value was $1.5$ and the minimum value was $-0.5$. We set the safety threshold $h$ to $-0.25$ in consideration of the required electrical power.

\begin{figure}[t]
	\begin{tabular}{cccc}
	\begin{minipage}{.15\textwidth}
		\centering
		\includegraphics[width=22mm]{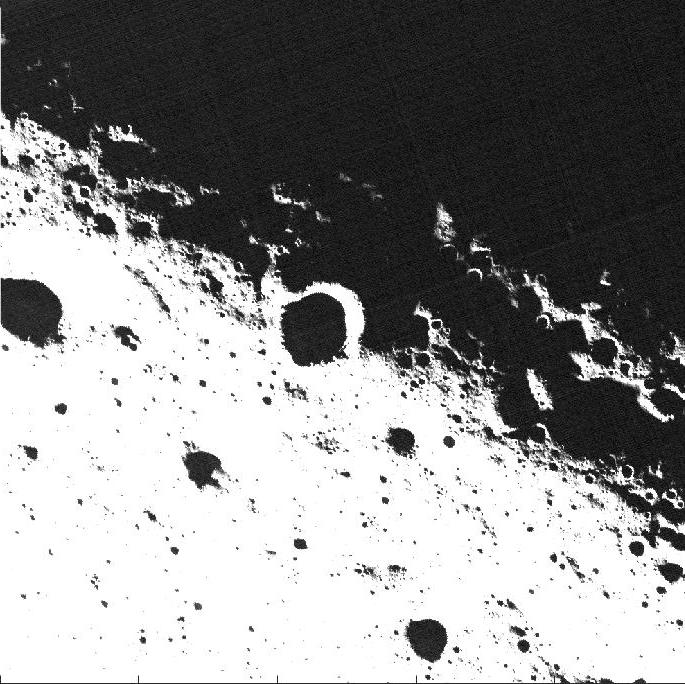}
	\end{minipage}
	\hspace{-3mm}
		\begin{minipage}{.15\textwidth}
		\centering
		\includegraphics[width=22mm]{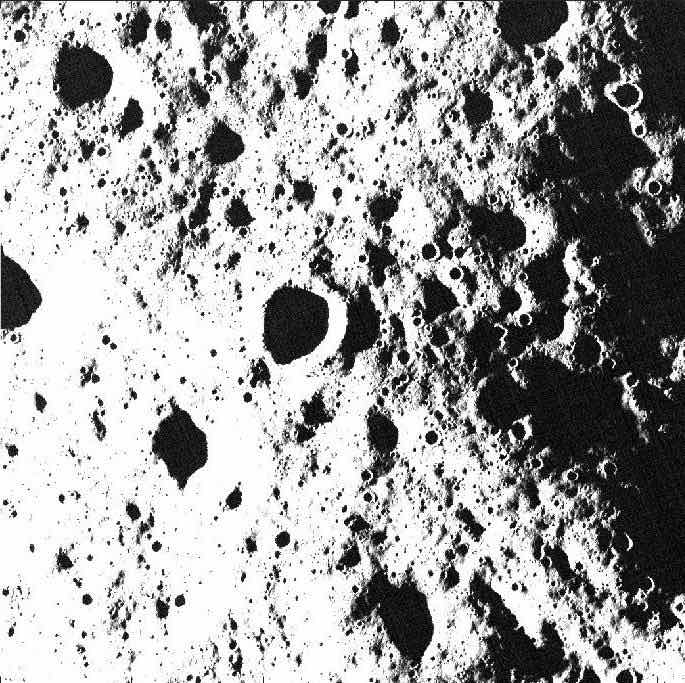}
	\end{minipage}
	\hspace{-3mm}
		\begin{minipage}{.15\textwidth}
		\centering
		\includegraphics[width=22mm]{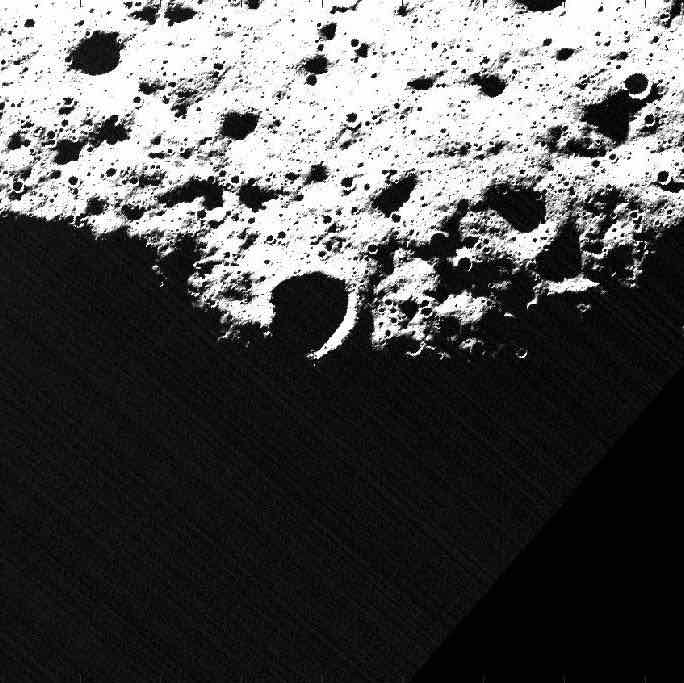}
	\end{minipage}
	\end{tabular}
    \caption{\label{fig:nac_pic}Three out of five images of the same region (1.0 km $\times$ 1.0 km) centered on \textit{de Gerlache Rim} (88.664$^\circ$ S, 68.398$^\circ$ W). Spatio-temporal map with 200 time steps was created by linearly interpolating the five images.}
\end{figure}

The luminance value of each state was predicted using a GP and we used a spatio-temporal kernel defined as (\ref{eq:st_kernel}). We used RBF for all kernel functions. The length-scales for $k_s$, $k_t$, $\hat{k}_s$, and $\hat{k}_t$ were 20 m, 9 hours, 30 m, and 36 hours, respectively. The prior variances for each kernel were 8 m$^2$, 4.5 hours$^2$, 4 m$^2$, and 2.4 hours$^2$, respectively. We assumed a noise standard deviation of 0.001. The parameter settings were $\beta_t=2\, \forall t \ge 0$, $p = 3$\, $L_s=0.05$, and $L_t=0.01$.

The five metrics in the first simulation were again used. As shown in Table~\ref{tab:lunar_terrain}, our algorithm outperformed the four baselines in terms of accuracy, precision, and recall. Our problem formulation aims to enlarge recall while keeping precision quite high. Higher precision and recall mean that \algo\ efficiently expanded the safe region by solving the min-max problem while avoiding FP. As shown in Figure~\ref{fig:simulation_result}, exploration was the safest and most efficient with \algo.

\begin{figure}[t]
\centering
\vspace{-1mm}
\hspace{2mm}
	\includegraphics[width=73mm]{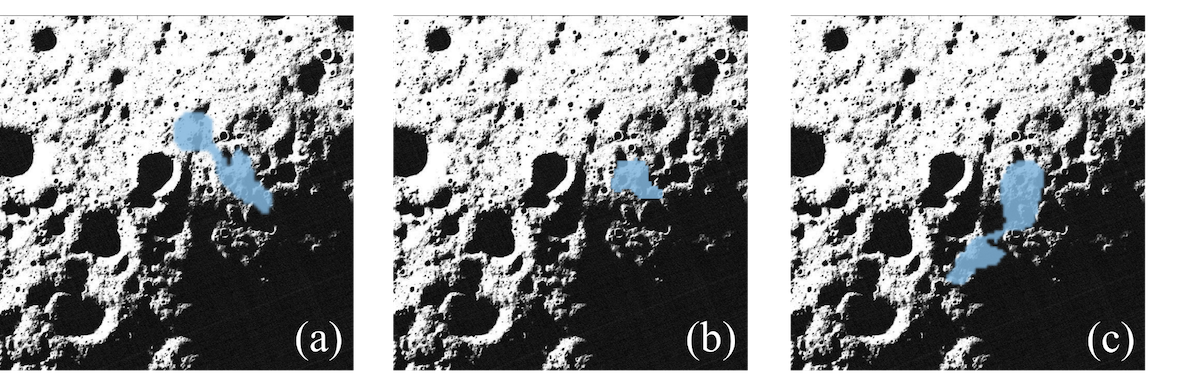}
	\vspace{-2mm}
	\caption{Comparison of different exploration schemes. Blue regions represent the (predicted) safe spaces for (a) \algo, (b)~Ignore time-variance case, and the explored space for (c)~Unsafe case.}
	\label{fig:simulation_result}
	\vspace{-1mm}
\end{figure}

\section{Conclusion}

We introduced the novel problem of safely exploring a time-variant and a priori unknown environment. We formulated it as a min-max problem in which the cumulative number of safe states is maximized and then applied our proposed algorithm, \algo, which efficiently expands the safe region. An important aspect of this algorithm is that it explicitly considers the lower bound of the cumulative number of safe states and then chooses the state with the maximum value of the weighted summation of the mean and the width of the confidence interval of the safety function. This algorithm incentivizes exploration while guaranteeing safety by exploiting the GP structure of the safety function. Finally, we demonstrated the effectiveness of our proposed approach by theoretical analysis and numerical simulation. 

Moving forward, it would be interesting to extend to the reinforcement learning setting in which the cumulative reward is maximized, as well as apply our proposed approach to other problem domains including medical care and robotic manipulation tasks. We believe that our results contribute to bridging the gap between machine learning algorithms and safety-critical applications.

\bibliographystyle{aaai}
\bibliography{reference}

\begin{thebibliography}{}

\bibitem[\protect\citeauthoryear{Berkenkamp \bgroup et al\mbox.\egroup
  }{2017}]{berkenkamp2017safe}
Berkenkamp, F.; Turchetta, M.; Schoellig, A.; and Krause, A.
\newblock 2017.
\newblock Safe model-based reinforcement learning with stability guarantees.
\newblock In {\em Neural Information Processing Systems (NIPS)}.

\bibitem[\protect\citeauthoryear{Blackmore \bgroup et al\mbox.\egroup
  }{2010}]{blackmore2010probabilistic}
Blackmore, L.; Ono, M.; Bektassov, A.; and Williams, B.~C.
\newblock 2010.
\newblock A probabilistic particle-control approximation of chance-constrained
  stochastic predictive control.
\newblock {\em IEEE transactions on Robotics} 26:502--517.

\bibitem[\protect\citeauthoryear{Duvenaud \bgroup et al\mbox.\egroup
  }{2013}]{duvenaud2013structure}
Duvenaud, D.; Lloyd, J.; Grosse, R.; Tenenbaum, J.; and Zoubin, G.
\newblock 2013.
\newblock Structure discovery in nonparametric regression through compositional
  kernel search.
\newblock In {\em International Conference on Machine Learning (ICML)}.

\bibitem[\protect\citeauthoryear{Fleming and McEneaney}{1995}]{fleming1995risk}
Fleming, W.~H., and McEneaney, W.~M.
\newblock 1995.
\newblock Risk-sensitive control on an infinite time horizon.
\newblock {\em SIAM Journal on Control and Optimization} 33(6):1881--1915.

\bibitem[\protect\citeauthoryear{Garc{\i}a and
  Fern{\'a}ndez}{2015}]{garcia2015comprehensive}
Garc{\i}a, J., and Fern{\'a}ndez, F.
\newblock 2015.
\newblock A comprehensive survey on safe reinforcement learning.
\newblock {\em Journal of Machine Learning Research (JMLR)} 16:1437--1480.

\bibitem[\protect\citeauthoryear{Ghosal and Roy}{2006}]{ghosal2006posterior}
Ghosal, S., and Roy, A.
\newblock 2006.
\newblock Posterior consistency of {Gaussian} process prior for nonparametric
  binary regression.
\newblock {\em The Annals of Statistics} 34:2413--2429.

\bibitem[\protect\citeauthoryear{Hans \bgroup et al\mbox.\egroup
  }{2008}]{hans2008safe}
Hans, A.; Schneega{\ss}, D.; Sch{\"a}fer, A.~M.; and Udluft, S.
\newblock 2008.
\newblock Safe exploration for reinforcement learning.
\newblock In {\em European Symposium on Artificial Neural Networks (ESANN)}.

\bibitem[\protect\citeauthoryear{Hartikainen, Riihim{\"a}ki, and
  S{\"a}rkk{\"a}}{2011}]{hartikainen2011sparse}
Hartikainen, J.; Riihim{\"a}ki, J.; and S{\"a}rkk{\"a}, S.
\newblock 2011.
\newblock Sparse spatio-temporal gaussian processes with general likelihoods.
\newblock In {\em International Conference on Artificial Neural Networks
  (ICANN)}.

\bibitem[\protect\citeauthoryear{Krause and Ong}{2011}]{krause2011contextual}
Krause, A., and Ong, C.~S.
\newblock 2011.
\newblock Contextual gaussian process bandit optimization.
\newblock In {\em Neural Information Processing Systems (NIPS)}.

\bibitem[\protect\citeauthoryear{Moldovan and Abbeel}{2012}]{moldovan2012safe}
Moldovan, T., and Abbeel, P.
\newblock 2012.
\newblock Safe exploration in {Markov} decision processes.
\newblock In {\em International Conference on Machine Learning (ICML)}.

\bibitem[\protect\citeauthoryear{Rasmussen and
  Williams}{2006}]{rasmussen2006gaussian}
Rasmussen, C.~E., and Williams, C.~K.
\newblock 2006.
\newblock {\em Gaussian processes for machine learning}.
\newblock MIT Press.

\bibitem[\protect\citeauthoryear{S{\"a}rkk{\"a} and
  Hartikainen}{2012}]{sarkka2012infinite}
S{\"a}rkk{\"a}, S., and Hartikainen, J.
\newblock 2012.
\newblock Infinite-dimensional kalman filtering approach to spatio-temporal
  gaussian process regression.
\newblock In {\em International Conference on Artificial Intelligence and
  Statistics (AISTATS)}.

\bibitem[\protect\citeauthoryear{Schwarm and
  Nikolaou}{1999}]{schwarm1999chance}
Schwarm, A.~T., and Nikolaou, M.
\newblock 1999.
\newblock Chance-constrained model predictive control.
\newblock {\em AIChE Journal} 45(8):1743--1752.

\bibitem[\protect\citeauthoryear{Senanayake, Simon~Timothy, and
  Ramos}{2016}]{senanayake2016predicting}
Senanayake, R.; Simon~Timothy, O.; and Ramos, F.
\newblock 2016.
\newblock Predicting spatio-temporal propagation of seasonal influenza using
  variational gaussian process regression.
\newblock In {\em {AAAI} Conference on Artificial Intelligence ({AAAI})}.

\bibitem[\protect\citeauthoryear{Soh, Su, and Demiris}{2012}]{soh2012online}
Soh, H.; Su, Y.; and Demiris, Y.
\newblock 2012.
\newblock Online spatio-temporal gaussian process experts with application to
  tactile classification.
\newblock In {\em International Conference on Intelligent Robots and Systems
  (IROS)}.

\bibitem[\protect\citeauthoryear{Srinivas \bgroup et al\mbox.\egroup
  }{2010}]{srinivas2009gaussian}
Srinivas, N.; Krause, A.; Kakade, S.~M.; and Seeger, M.
\newblock 2010.
\newblock Gaussian process optimization in the bandit setting: No regret and
  experimental design.
\newblock In {\em International Conference on Machine Learning (ICML)}.

\bibitem[\protect\citeauthoryear{Sui \bgroup et al\mbox.\egroup
  }{2015}]{sui2015safe}
Sui, Y.; Gotovos, A.; Burdick, J.~W.; and Krause, A.
\newblock 2015.
\newblock Safe exploration for optimization with gaussian processes.
\newblock In {\em International Conference on Machine Learning (ICML)}.

\bibitem[\protect\citeauthoryear{Sui, Yue, and
  Burdick}{2017}]{sui2017correlational}
Sui, Y.; Yue, Y.; and Burdick, J.~W.
\newblock 2017.
\newblock Correlational dueling bandits with application to clinical treatment
  in large decision spaces.
\newblock In {\em International Joint Conference on Artificial Intelligence
  (IJCAI)}.

\bibitem[\protect\citeauthoryear{Turchetta, Berkenkamp, and
  Krause}{2016}]{turchetta2016safe}
Turchetta, M.; Berkenkamp, F.; and Krause, A.
\newblock 2016.
\newblock Safe exploration in finite {Markov} decision processes with gaussian
  processes.
\newblock In {\em Neural Information Processing Systems (NIPS)}.

\bibitem[\protect\citeauthoryear{Wachi \bgroup et al\mbox.\egroup
  }{2018}]{wachi2018safe}
Wachi, A.; Sui, Y.; Yue, Y.; and Ono, M.
\newblock 2018.
\newblock Safe exploration and optimization of constrained mdps using gaussian
  processes.
\newblock In {\em {AAAI} Conference on Artificial Intelligence ({AAAI})}.

\end{thebibliography}

\onecolumn

\maketitle

{\huge \section{Supplemental Material}}

\vspace{5mm}
\subsection{A. Preliminary Lemmas}

\begin{lemma}
\label{lem:reach_ret}
For any $\mathcal{X} \subseteq \mathcal{S}$, the following hold:
\begin{enumerate}[label=(\roman*)]
\setlength{\leftskip}{0.7cm}
\item $R_{\text{reach}}(\mathcal{X}) \supseteq \mathcal{X}$,
\item $R_{\text{ret}}(\mathcal{X}) \supseteq \mathcal{X}$.
\end{enumerate}
\end{lemma}

\begin{proof}
This lemma follows directly from the definitions of $R_{\text{reach}}$ and $R_{\text{ret}}$.
\end{proof}

\begin{lemma}
\label{lem:size_observation}
For any $j \ge i \ge 1$, the following hold.
\begin{enumerate}[label=(\roman*)]
\setlength{\leftskip}{0.7cm}
\item $S_i(\bm{y}_{i-1}) = S_i(\bm{y}_j)$,
\item $G_{i-1}^{i+1}(\bm{y}_{i-1}) = G_{i-1}^{i+1}(\bm{y}_{j})$,
\item $\hat{S}_i(\bm{y}_{i-1}) = \hat{S}_i(\bm{y}_j)$.
\end{enumerate}
\end{lemma}

\begin{proof}
This lemma follows directly from the definitions of $S$, $G$ and $\hat{S}$; $S_i$, $G_{i-1}^{i+1}$, and $\hat{S}_i$ depend only on $\bm{y}_{i-1}$.
\end{proof}

\begin{lemma}
\label{lem:S_G}
For any $i \ge j \ge 1$, the following hold:
\begin{enumerate}[label=(\roman*)]
\setlength{\leftskip}{0.7cm}
\item $S_i(\bm{y}_j) \supseteq G_{i-1}^{i+1}(\bm{y}_j)$,
\item $G_{i}^{i+2}(\bm{y}_j) \cap G_{i+1}^{i+3}(\bm{y}_j) = G_{i}^{i+3}(\bm{y}_j)$.
\end{enumerate}
\end{lemma}
\begin{proof}
This lemma follows directly from the definition of $S$ and $G$.
\end{proof}

\begin{lemma}
\label{lem:bound}
For any $t \in \mathcal{T}$ and $\eta \in \mathcal{H}$, the following hold:
\begin{enumerate}[label=(\roman*)]
\setlength{\leftskip}{0.7cm}
\item $l_{t+1}(\eta) \ge l_{t}(\eta)$,
\item $u_{t+1}(\eta) \le u_{t}(\eta)$.
\end{enumerate}
\end{lemma}
\begin{proof}
This lemma follows directly from the definitions of $l$ and $u$.
\end{proof}

\subsection{B. Safe Space}

\begin{lemma}
\label{lemma:S_observation}
For any $t$, $|S_{t+1}(\bm{y}_t)| \ge |S_{t+1}(\bm{y}_{t-1})|$.
\end{lemma}

\begin{proof}
By definition of $S$, 
\begin{alignat*}{2}
S_{t+1}(\bm{y}_t) =&\ \{ \bm{s} \in \mathcal{S} \mid \ \exists \bm{s}' \in \hat{S}_{t}(\bm{y}_t): l_{t+1}(\{t,\bm{s}'\}) - L(\bm{s},\bm{s}',t+1, t) \ge h \}, \\
S_{t+1}(\bm{y}_{t-1}) =&\ \{ \bm{s} \in \mathcal{S} \mid \ \exists \bm{s}' \in \hat{S}_{t}(\bm{y}_{t-1}): l_{t}(\{t,\bm{s}'\}) - L(\bm{s},\bm{s}',t+1, t) \ge h \}.
\end{alignat*}
By Lemma~\ref{lem:size_observation} (iii), it holds that $\hat{S}_t(\bm{y}_t) = \hat{S}_t(\bm{y}_{t-1})$.
In addition, by Lemma~\ref{lem:bound} (i), $l_{t+1}(\{t, \bm{s}'\}) \ge l_{t}(\{t, \bm{s}'\})$ is satisfied for the lower bound. Hence, the following holds:
\[
S_{t+1}(\bm{y}_t) \supseteq S_{t+1}(\bm{y}_{t-1}).
\]
\end{proof}

\begin{lemma}
\label{lemma:G_observation}
For any $t$, $|G_{t}^{t+2}(\bm{y}_t)| \ge |G_t^{t+2}(\bm{y}_{t-1})|$.
\end{lemma}

\begin{proof}
By definition of $G$, 
\begin{alignat*}{2}
G_{t}^{t+2}(\bm{y}_{t}) =&\ \{ \bm{s} \in \mathcal{S} \mid \exists \bm{s}' \in \hat{S}_{t}(\bm{y}_t): 
l_{t+1}(\{t,\bm{s}'\}) - L(\bm{s},\bm{s}', t+2, t) \ge h \}, \\
G_{t}^{t+2}(\bm{y}_{t-1}) =&\ \{ \bm{s} \in \mathcal{S} \mid \exists \bm{s}' \in \hat{S}_{t}(\bm{y}_{t-1}): 
l_t(\{t,\bm{s}'\}) - L(\bm{s},\bm{s}', t+2, t) \ge h \}.
\end{alignat*}
By Lemma~\ref{lem:size_observation} (iii), it holds that $\hat{S}_t(\bm{y}_t) = \hat{S}_t(\bm{y}_{t-1})$.
In addition, by Lemma~\ref{lem:bound} (i), $l_{t+1}(\{t, \bm{s}'\}) \ge l_{t}(\{t, \bm{s}'\})$ is satisfied for the lower bound. Hence, the following holds:
\[
G_{t}^{t+2}(\bm{y}_t) \supseteq G_t^{t+2}(\bm{y}_{t-1}).
\]
\end{proof}

\begin{proof}\!\text{[of \textbf{Lemma~\ref{theorem_observation}}]} Consider $i=t+2$, then the following holds:
\begin{equation*}
\begin{split}
q(t+2, \bm{s}_{t+2}; \bm{y}_{t+1})
= &\ |\hat{S}_{t+2}(\bm{y}_{t+1})| \\
= &\ |S_{t+2}(\bm{y}_{t+1}) \cap R_\text{reach}(\hat{S}_{t+1}(\bm{y}_{t+1})) \cap R_\text{ret}(G_{t+1}^{t+3}(\bm{y}_{t+1}))| \\
= &\ |S_{t+2}(\bm{y}_{t+1}) \cap R_\text{reach}(\hat{S}_{t+1}(\bm{y}_{t})) \cap R_\text{ret}(G_{t+1}^{t+3}(\bm{y}_{t+1}))|  \hspace{1cm} \text{(By Lemma~\ref{lem:size_observation} (iii))}\\
\ge &\ |S_{t+2}(\bm{y}_{t}) \cap R_\text{reach}(\hat{S}_{t+1}(\bm{y}_{t})) \cap R_\text{ret}(G_{t+1}^{t+3}(\bm{y}_{t}))| \hspace{1.7cm} \text{(By Lemmas~\ref{lemma:S_observation} and \ref{lemma:G_observation})} \\
= &\ |\hat{S}_{t+2}(\bm{y}_{t})| \\
= &\ q(t+2, \bm{s}_{t+2}; \bm{y}_{t}).
\end{split}
\end{equation*}
The above inequality is recursively satisfied for all $i = [t+2, N]$; that is, 
\begin{alignat*}{2}
q(i, \bm{s}_{i}; \bm{y}_{i-1}) \ge q(i, \bm{s}_{i}; \bm{y}_{i-2}) \ge \cdots \ge q(i, \bm{s}_{i}; \bm{y}_{t}).
\end{alignat*}
Therefore, from the above inequality, 
\begin{alignat*}{2}
J(t, \bm{s}_t; \bm{y}_{i-1})
= &\ \sum_{i=t}^N q(i, \bm{s}_i; \bm{y}_{i-1}) \\
\ge &\ q(t, \bm{s}_t; \bm{y}_{t-1}) + \sum_{i=t+1}^N q(i, \bm{s}_i; \bm{y}_{t}).
\end{alignat*}
Hence, Lemma~\ref{theorem_observation} holds.
\end{proof}

\begin{lemma}
\label{lem:hat_S_recursive}
For all $i = [t, N]$ and $j<i$, the following inequality holds:
\[
|\hat{S}_i(\bm{y}_j)| \ge |\hat{S}_{i-1}(\bm{y}_j) \cap G_{i+1}(\bm{y}_j)|.
\]
\end{lemma}

\begin{proof}
From definition of $\hat{S}_i(\bm{y}_j)$,
\begin{alignat*}{2}
\hat{S}_i(\bm{y}_j) = S_i(\bm{y}_j) \cap R_\text{reach}(\hat{S}_{i-1}(\bm{y}_j)) \cap R_\text{ret}(G_{i-1}^{i+1}(\bm{y}_j)).
\end{alignat*}
From Lemmas~\ref{lem:reach_ret} and \ref{lem:S_G} (i),
\begin{alignat*}{2}
\hat{S}_i(\bm{y}_j) 
& \supseteq S_i(\bm{y}_j) \cap \hat{S}_{i-1}(\bm{y}_j) \cap G_{i-1}^{i+1}(\bm{y}_j) \hspace{1cm} \text{(By Lemma \ref{lem:reach_ret})} \\
& = \hat{S}_{i-1}(\bm{y}_j) \cap G_{i-1}^{i+1}(\bm{y}_j). \hspace{2.3cm} \text{(By Lemma \ref{lem:S_G} (i))}
\end{alignat*}
Therefore, Lemma~\ref{lem:hat_S_recursive} holds.
\end{proof}

\begin{proof}\!\text{[of \textbf{Lemma~\ref{lem:hat_S_recursive2}}]} 
By Lemma~\ref{lem:hat_S_recursive}, 
\begin{alignat*}{2}
\hat{S}_i(\bm{y}_t) \supseteq \hat{S}_{i-1}(\bm{y}_t) \cap G_{i-1}^{i+1}(\bm{y}_t).
\end{alignat*}
By using the above inequality recursively, we have the following
\begin{alignat*}{2}
\hat{S}_i(\bm{y}_t) 
& \supseteq \hat{S}_{i-1}(\bm{y}_t) \cap G_{i-1}^{i+1}(\bm{y}_t) \\
& \supseteq (\hat{S}_{i-2}(\bm{y}_t) \cap G_{i-2}^{i}(\bm{y}_t)) \cap G_{i-1}^{i+1}(\bm{y}_t) \\
& = \hat{S}_{i-2}(\bm{y}_t) \cap G_{i-2}^{i+1}(\bm{y}_t) \hspace{3cm} \text{(By Lemma~\ref{lem:S_G} (ii))} \\
& \supseteq \cdots \\
& \supseteq \hat{S}_{t}(\bm{y}_t) \cap G_t^{i+1}(\bm{y}_t).
\end{alignat*}
Therefore, Lemma~\ref{lem:hat_S_recursive2} holds.
\end{proof}

\begin{proof}\!\text{[of \textbf{Theorem~\ref{theorem:lower_bound_J}}]} 
By Lemma~\ref{lem:hat_S_recursive2}, 
\begin{alignat*}{2}
|\hat{S}_i(\bm{y}_t)| \ge |\hat{S}_{t}(\bm{y}_t) \cap G_t^{i+1}(\bm{y}_t)|.
\end{alignat*}
Hence, using Lemma~\ref{theorem_observation}, we have the following
\begin{alignat*}{2}
J(t, \bm{s}_t; \bm{y}_{t-1}) 
&\ge |\hat{S}_t(\bm{y}_{t-1})| + \sum_{i=t+1}^{N}  |\hat{S}_{t}(\bm{y}_t) \cap G_{t}^{i+1}(\bm{y}_t)| \\
&\ge |\hat{S}_t(\bm{y}_{t-1})| + \sum_{i=t+1}^{N}  M_i(\bm{y}_t)
\label{eq:V_M}
\end{alignat*}
where $M_i$ is defined as
$M_i := |\hat{S}_{t} \cap G_{t}^{i+1}|$. 
Therefore, Theorem~\ref{theorem:lower_bound_J} holds.
\end{proof}

\subsection{C. Bounding of Information Gain of the Spatio-Temporal Kernel}

\begin{lemma}
\label{lemma:product_kernel_bound_IG}
Let $k_t$ be a kernel function on $\mathcal{T}$ with rank at most $d$. Then,
\[
\gamma(T; k_s \otimes k_t; X) \le d \gamma(T; k_s; \mathcal{S}) + d \log T.
\]
\end{lemma}

\begin{proof}
This lemma is a direct result from Theorem 2 of \citet{krause2011contextual}.
\end{proof}

\begin{lemma}
\label{lemma:additive_kernel_bound_IG}
Let $k_s$ and $k_t$ be kernel functions on $\mathcal{S}$ and $\mathcal{T}$, respectively. Then, for the additive combination $k = k_s \oplus k_t$ defined on $X$, it holds that
\[
\gamma(T; k_s \oplus k_t; X) \le \gamma(T; k_s; \mathcal{S}) + \gamma(T; k_t; \mathcal{T}) + 2\log T.
\]
\end{lemma}

\begin{proof}
This lemma is a direct result from Theorem 3 of \citet{krause2011contextual}.
\end{proof}

\begin{proof} \text{[of \textbf{Lemma~\ref{upper_IG}}]}
For any pair of two kernels, $k_1$ and $k_2$, the information gain of the composite kernel, $k_{12}$ defined on $X_{12}$ is bounded by
\[
\gamma(T; k_{12}; X_{12}) = \max \{ \gamma(T; k_1 \oplus k_2; X_{12}^\oplus), \gamma(T; k_1 \otimes k_2; X_{12}^\otimes) \},
\]
where $X_{12}^\oplus$ and $X_{12}^\otimes$ represent the spaces for which $k_1 \oplus k_2$ and $k_1 \otimes k_2$ are defined, respectively. By the above relationship, $k_F(k_s^1, k_s^2, \cdots, k_t^1, k_t^2, \cdots)$ is guaranteed to have a bound on the information gain.

For example, a composite kernel is defined as $k_F = (k_s \oplus k_t) \oplus (\hat{k}_s \otimes \hat{k}_t)$, where $k_s$ and $\hat{k}_s$ are kernels on $\mathcal{S}$, and $k_t$ and $\hat{k}_t$ are kernels on $\mathcal{T}$. Let $d$ be the rank (see Lemma~\ref{lemma:product_kernel_bound_IG}). The information gain is:
\begin{alignat*}{2}
\gamma(T; k_F) 
&= \gamma(T; (k_s \oplus k_t) \oplus (\hat{k}_s \otimes \hat{k}_t)) \\
&\le \gamma(T; (k_s \oplus k_t)) + \gamma(T; \hat{k}_s \otimes \hat{k}_t) + 2 \log T \\
&\le \{\gamma(T; k_s; \mathcal{S}) + \gamma(T; k_t; \mathcal{T}) + 2\log T\} + \{d \gamma(T; k_s; \mathcal{S}) + d \log T\} + 2 \log T \\
&\le (d+1)\gamma(T; k_s; \mathcal{S}) + \gamma(T; k_t; \mathcal{T}) + (d+4)\log T.
\end{alignat*}
\end{proof}

\subsection{D. Probabilistic Safety Guarantee}

\begin{lemma}
\label{pr_g_ge_l}
If the next state is chosen in $S_t$, the probability of entering a safe state is guaranteed to be at least 
\[
\Pr[g(t, \bm{s}) \ge l(t, \bm{s})].
\]
\end{lemma}

\begin{proof}
The probability of entering safe state is denoted as
\[
\Pr[g(t, \bm{s})\ge h \mid l(t', \bm{s}')-L_s d(\bm{s}, \bm{s}') - L_t d(\bm{t}, \bm{t}') \ge h].
\]
Under the assumption of Lipschitz continuity with Lipschitz constants $L_s$ and $L_t$, we transform the above formula:
\begin{alignat}{6}
\label{eq: Probability_guarantee}
\nonumber
&\ \Pr[g(t, \bm{s})\ge h \mid l(t', \bm{s}')-L_s d(\bm{s}, \bm{s}') - L_t d(\bm{t}, \bm{t}') \ge h] \\
\nonumber
\ge &\ \Pr[g(t, \bm{s})\ge h \mid l(t', \bm{s}) - L_t d(\bm{t}, \bm{t}') \ge h] \\
\nonumber
\ge &\ \Pr[g(t, \bm{s})\ge h \mid l(t, \bm{s})  \ge h] \\
\nonumber
\ge &\ \Pr[g(t, \bm{s}) \ge l(t, \bm{s}) \mid l(t, \bm{s})\ge h] \\
\nonumber
= &\ \Pr[g(t, \bm{s}) \ge l(t, \bm{s})].
\end{alignat}
The last equality exploits the conditional independence. Note that this probability can be tuned by adjusting $\beta$.
\end{proof}

\begin{lemma}
\label{srinvas_theo6}

Suppose that $\|g\|^2_k \le B$ and that noise $n_t$ is $\sigma$-sub-Gaussian. If $\beta_t$ is chosen as in (4),
then, for all $t \ge 1$ and all $\bm{s} \in \mathcal{S}$, it holds with probability
at least $1-\delta$ that $g(t, \bm{s}) \in C_t(\bm{\eta})$.
\end{lemma}

\begin{proof}
This lemma follows from Theorem 6 of \citet{srinivas2009gaussian}.
\end{proof}

\begin{proof}\!\text{[of \textbf{Theorem~\ref{theorem:safe_guarantee}}]}
By the definition of $l_t$; i.e., $l_t(\bm{\eta}) = \min C_t(\bm{\eta})$,
\[
\Pr[g(t, \bm{s}) \ge l(t, \bm{s})] \ge \Pr[g(t, \bm{s}) \in C_t(\bm{\eta})].
\]
Because of Lemma~\ref{pr_g_ge_l} and Lemma~\ref{srinvas_theo6},
\[
\Pr[g(t, \bm{s}) \ge l(t, \bm{s})]  \ge \Pr[g(t, \bm{s}) \in C_t(\bm{\eta})]　\ge 1-\delta.
\]
Hence, the probability that the agent takes an unsafe action is guaranteed to be at least $1-\delta$.
\end{proof}

\newpage
\subsection{E. Additional Information on the Simulation using Real Lunar Data}

Recall that we first prepared the total five images (see Figure~\ref{fig:nac_pic_appendix}) to create the spatio-temporal safety function. In this section, we provide the additional information on the data used for the simulated moon surface exploration especially in terms of how to prepare the five images.

We used the lunar images that were actually taken by Lunar Reconnaissance Orbiter (LRO); hence, the available number of images for the same region on the Moon is quite limited. That is why we created the spatio-temporal map by the linearly interpolating the images.

\begin{figure*}[h]
	\begin{tabular}{cccc}
    \begin{minipage}{.19\textwidth}
		\centering
		\includegraphics[width=25mm]{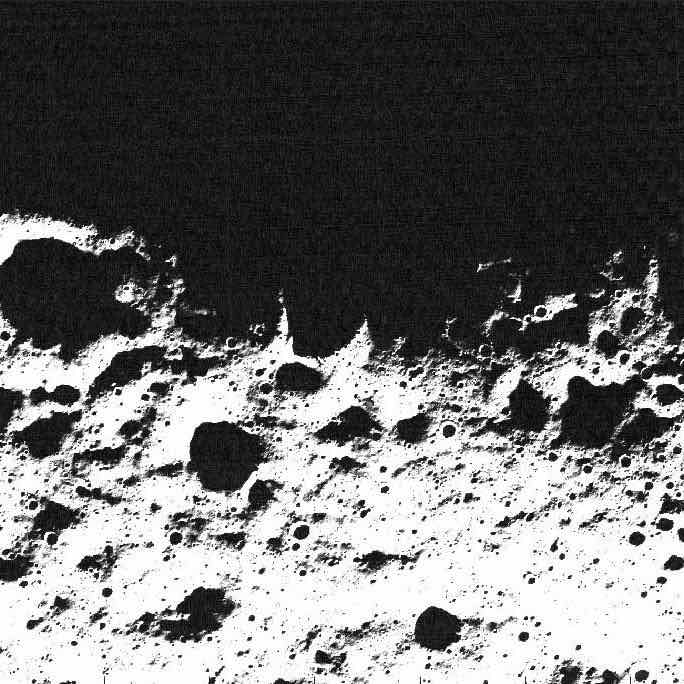}
	\end{minipage}
	\begin{minipage}{.19\textwidth}
		\centering
		\includegraphics[width=25mm]{figures/2.jpg}
	\end{minipage}
		\begin{minipage}{.19\textwidth}
		\centering
		\includegraphics[width=25mm]{figures/5.jpg}
	\end{minipage}
		\begin{minipage}{.19\textwidth}
		\centering
		\includegraphics[width=25mm]{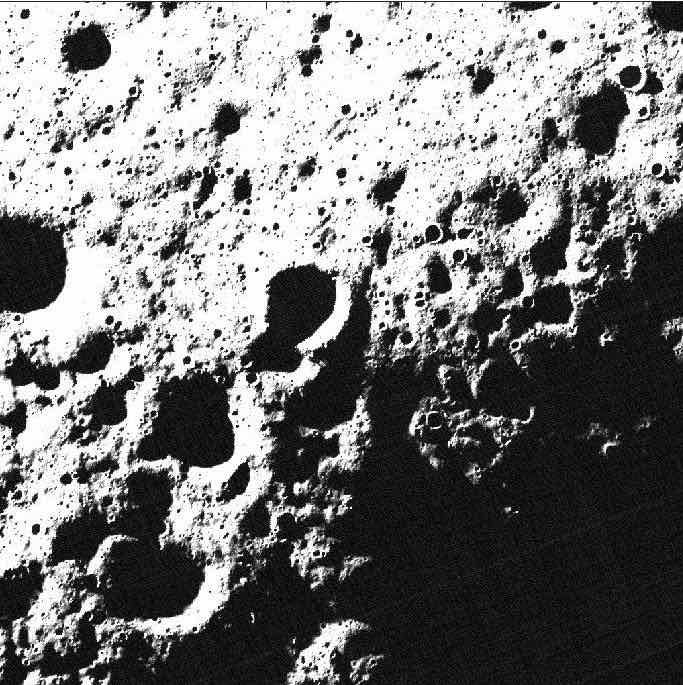}
	\end{minipage}
		\begin{minipage}{.19\textwidth}
		\centering
		\includegraphics[width=25mm]{figures/7.jpg}
	\end{minipage}
	\end{tabular}
    \caption{Five images of the same region (1.0 km $\times$ 1.0 km) centered on \textit{de Gerlache Rim} (88.664$^\circ$ S, 68.398$^\circ$ W). Spatio-temporal map with 200 time steps was created by linearly interpolating the images.}
\label{fig:nac_pic_appendix}
\end{figure*}

We first downloaded the six images with the following IDs from the Lunar Reconnaissance Orbiter Camera (LROC) website,\footnote{\url{https://goo.gl/6eXmQP}} and then clipped the same region (1.0 km $\times$ 1.0 km) centered on \textit{de Gerlache rim} (88.664$^\circ$ S, 68.398$^\circ$ W) from them.

\begin{itemize}
\setlength{\leftskip}{0.7cm}
    \item Image\_1: M138264280RE
    \item Image\_2: M140625419LE + M140625419RE
    \item Image\_3: M143000050LE
    \item Image\_4: M180899951LE
    \item Image\_5: M1100039834RE
\end{itemize}

In M140625419LE and M140625419RE, since all the targeted region is not in the angle of view as shown in Figure~\ref{fig:NAC_above_below}, we integrated with each other. Note that the time stamps of the two images are identical (i.e., M140625419LE and M140625419RE were taken by the left and right cameras of LRO almost at the same time).

\begin{figure}[h]
\vspace{5mm}
\centering
\subfigure[M140625419LE]{\includegraphics[width=2.8cm]{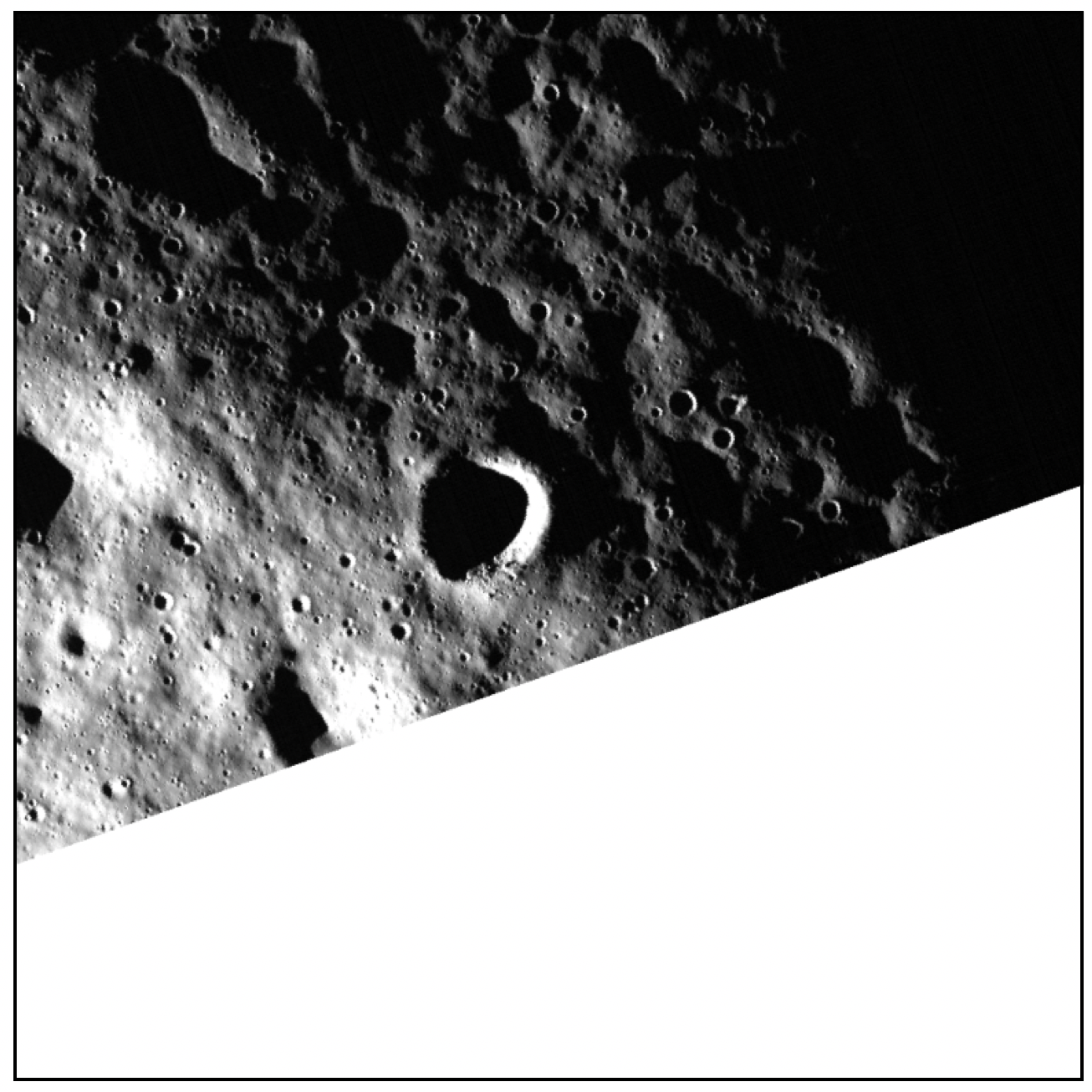}}
\hspace{25mm}
\subfigure[M140625419RE]{\includegraphics[width=2.8cm]{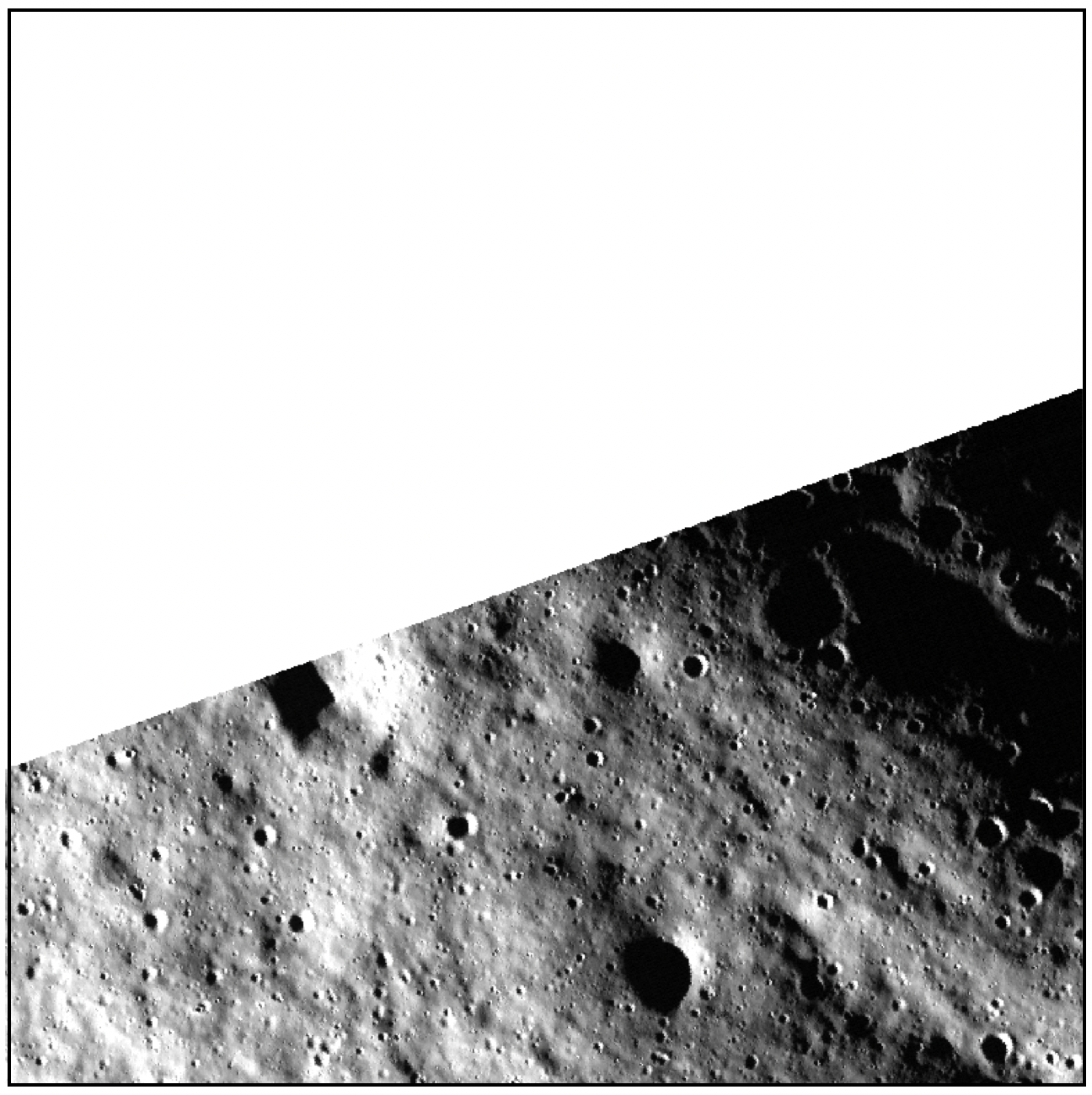}}
\caption{In the above two images, all the target region is not in the angle of view. Hence, we merged them into one image (i.e., Image\_2).}
\label{fig:NAC_above_below}
\end{figure}

Finally, the simulated environment was created by first obtaining five images of the target region and then linearly interpolating them so that the duration between images was $\sim \!$ 1.8 hours. The overall mission period was assumed to be $\sim \!$ 15 days, so the number of images was 200. 
\end{document}